%% file: main.tex
\documentclass{article}

\usepackage[preprint]{neurips_2026} 

\usepackage{amsmath, amssymb, amsthm}
\usepackage{graphicx}
\usepackage{hyperref}

\usepackage[utf8]{inputenc} 
\usepackage[T1]{fontenc}    
\usepackage{url}            
\usepackage{booktabs}       
\usepackage{amsfonts}       
\usepackage{nicefrac}       
\usepackage{microtype}      
\usepackage{xcolor}         
\usepackage{amsthm}
\usepackage{enumitem}
\usepackage{subfigure}
\usepackage{xspace}
\usepackage{wrapfig}
\usepackage{longtable}
\usepackage[table,dvipsnames]{xcolor}
\usepackage{algorithm}
\usepackage{algpseudocode}
\usepackage{amsmath, amssymb}
\usepackage{cleveref}
\usepackage{multirow}
\usepackage{graphicx}
\usepackage{tikz}
\usepackage[framemethod=TikZ]{mdframed}
\usetikzlibrary{arrows.meta,backgrounds,calc,decorations.pathreplacing,fit,positioning}

\setcitestyle{numbers,square}

\newtheorem{theorem}{Theorem}[section]
\newtheorem{proposition}[theorem]{Proposition}
\newtheorem{lemma}[theorem]{Lemma}

\theoremstyle{definition}
\newtheorem{definition}[theorem]{Definition}

\theoremstyle{remark}

\mdfdefinestyle{observationbox}{
    linecolor=white,
    outerlinewidth=1pt,
    roundcorner=5pt,
    innertopmargin=4pt,      
    innerbottommargin=3pt,
    innerrightmargin=5pt,
    innerleftmargin=5pt,
    skipabove=6pt,           
    skipbelow=6pt,
    backgroundcolor=blue!6!white
}

\newenvironment{insight}[1]
{
\begin{center}
\begin{minipage}{0.9\linewidth}
\begin{mdframed}[style=observationbox]
}
{
\end{mdframed}
\end{minipage}
\end{center}
}
\surroundwithmdframed[style=observationbox]{observation}

\newif\ifshowcomments
\showcommentsfalse  

\newcommand{\ourmethod}{\textsc{KV-CAT}\xspace}

\input{math_commands}

\title{Training Transformers for KV Cache Compressibility}

\author{
  Yoav Gelberg$^{1}$\thanks{Equal contribution. Order determined by flipping a 2016-issued one-pound coin.} \quad
  Yam Eitan$^{2}$\footnotemark[1] \quad
  Michael Bronstein$^{1, 3}$ 
  Yarin Gal$^{1}$ \quad
  Haggai Maron$^{2,4}$ 
  \\
  \\
  $^{1}$University of Oxford \quad
  $^{2}$Technion -- Israel Institute of Technology \quad
  $^{3}$AITHYRA \quad
  $^{4}$NVIDIA
}

\date{April 2026}

\begin{document}

\maketitle

\begin{abstract}
\looseness=-1
Long-context language modeling is increasingly constrained by the Key--Value (KV) cache, whose memory and decode-time access costs scale linearly with the prefix length. This bottleneck has motivated a range of context-compression methods, from token-level summarization to recent optimization-based KV cache compression methods. These post-hoc methods operate on the KV cache of a fixed pretrained model, so their effectiveness is fundamentally limited by how well the model's internal representations can be compressed. In this work, we formalize the notion of KV compressibility and show that it is a property of the \emph{learned representations}, rather than of the context alone. We prove that almost any sequence-to-vector function admits both highly compressible and inherently non-compressible transformer implementations, highlighting the need to guide transformers toward compressible representations during training. Motivated by this, we propose \textbf{KV}-\textbf{C}ompression \textbf{A}ware \textbf{T}raining (\ourmethod), a continued pretraining procedure that incentivizes the emergence of compressible representations. We introduce a train-time KV sparsification policy that masks KV slots during training. This forces the model to use fewer KV slots and encourages it to learn representations amenable to post-hoc compression. Empirically, we show that \ourmethod improves the quality–budget tradeoff of downstream compression methods across retrieval, long-context question answering, and perplexity-based evaluation of compressed-prefix continuation.
\end{abstract}



\input{sections/introduction}
\input{sections/preliminaries}
\input{sections/motivation_and_theoretical_results}
\input{sections/method}
\input{sections/experiments}

\input{sections/related_works}
\input{sections/conclusion}
\input{sections/acks}

\bibliographystyle{plainnat}
\bibliography{references}

\appendix

\input{appendix/theory}
\input{appendix/method}
\input{appendix/extended_experimental_details}
\input{appendix/additional_experimental_results}
\input{appendix/limitations}
\input{appendix/broader_impact}

\end{document}

%% file: math_commands.tex
\usepackage{amsmath,amsfonts,bm}

\def\eqref#1{equation~\ref{#1}}

\def\1{\bm{1}}

\def\va{{\bm{a}}}
\def\vb{{\bm{b}}}

\def\ve{{\bm{e}}}

\def\vk{{\bm{k}}}

\def\vp{{\bm{p}}}
\def\vq{{\bm{q}}}

\def\vu{{\bm{u}}}
\def\vv{{\bm{v}}}
\def\vw{{\bm{w}}}
\def\vx{{\bm{x}}}
\def\vy{{\bm{y}}}
\def\vz{{\bm{z}}}

\def\mI{{\bm{I}}}

\def\mK{{\bm{K}}}

\def\mP{{\bm{P}}}
\def\mQ{{\bm{Q}}}

\def\mV{{\bm{V}}}
\def\mW{{\bm{W}}}
\def\mX{{\bm{X}}}
\def\mY{{\bm{Y}}}
\def\mZ{{\bm{Z}}}

\DeclareMathAlphabet{\mathsfit}{\encodingdefault}{\sfdefault}{m}{sl}
\SetMathAlphabet{\mathsfit}{bold}{\encodingdefault}{\sfdefault}{bx}{n}

\def\gC{{\mathcal{C}}}

\def\gL{{\mathcal{L}}}

\def\sN{{\mathbb{N}}}

\def\sR{{\mathbb{R}}}

\newcommand{\KL}{D_{\mathrm{KL}}}

\newcommand{\model}{\mathsf{M}} 

\newcommand{\compress}{\mathsf{C}} 
\newcommand{\cl}{\mathsf{c}} 

\newcommand{\length}{r} 
\newcommand{\ffn}{\mathrm{FFN}} 
\newcommand{\block}{\mathrm{Block}} 
\newcommand{\alb}{A} 
\newcommand{\te}{\mathrm{emb}} 
\newcommand{\func}{f} 
\newcommand{\const}{C} 
\newcommand{\conv}{H} 
\newcommand{\hist}{f_{\mathrm{hist}}} 

%% file: sections/introduction.tex
\section{Introduction}

\looseness=-1
Language models (LMs) are increasingly deployed in long-horizon settings, from understanding large codebases, long-form documents, and personal data repositories~\citep{nam2024using,islam2023financebench,arora2022can}, to long-form reasoning and continuously learning agents~\citep{chen2025towards,zheng2026lifelong,li2024larm}. Using autoregressive transformer-based LMs in these settings introduces a significant memory bottleneck: the Key--Value (KV) cache. During inference, these models must store the key and value vectors for every token, at every layer and (KV-)attention head. For long sequences, this cache can dominate both memory usage and decoding cost, turning context length into a primary serving bottleneck~\citep{zhang2023h2o,li2024snapkv,xiao2024streamingllm}.

\looseness=-1
A substantial body of recent work has sought to mitigate this bottleneck, broadly falling into two categories. The first category focuses on designing more efficient alternatives to transformers. Examples include linear attention mechanisms ~\citep{katharopoulos2020transformers,choromanski2020rethinking, horn2017delta, yang2024gated}, state space models~\citep{gu2021efficiently,gu2023mamba}, sparse attention variants~\citep{beltagy2020longformer,zaheer2020big}, and more. These methods significantly reduce the computational cost of long-context language modeling. However, this efficiency typically comes at the cost of empirical performance, and such models still lag behind transformers at scale.


The second category focuses on inference-time interventions applied to a fixed pretrained transformer. These methods either operate on the input context or, more generally, directly on the KV cache. Early approaches are largely heuristic, including textual summarization~\citep{antho, oai},  learned token filtering~\citep{jiang2023llmlingua}, and policies based on attention patterns, recency, heavy-hitter behavior, or layer-wise importance~\citep{zhang2023h2o,xiao2024streamingllm,li2024snapkv,cai2025pyramidkv}. More recently, optimization-based KV cache compression\footnote{These methods are also commonly referred to as KV compaction methods in the literature.} methods have emerged as powerful techniques. For example,~\citet{eyuboglu2025cartridges} use gradient-based optimization to match the distribution induced by the original cache, while~\citet{zweiger2026fastkv} employ layer-wise objectives to reproduce its attention traces.

The success of KV cache compression suggests that the full KV cache often contains redundancies. However, as most approaches operate on a \textbf{fixed model},
they are inherently limited by how well that particular model's representations can be compressed.
Crucially, this compressibility is not determined by the input sequence alone: two transformers can use very different internal representations when processing the same sequence, yet have identical next-token distributions. Consequently, some transformers' KV caches may be more amenable to compression than others.

\looseness=-1
\textbf{This work.} In this paper, we study the following question: can transformer LMs be trained in a way that leads to KV caches that are more amenable to post-hoc compression? This shifts the target from the KV cache compression algorithm to the model itself.
To address this question, we introduce the notion of {\bf KV-compressibility}. Informally, a transformer is KV-compressible if there exists a compression policy that maps the KV cache of long input sequences to a shorter KV cache while preserving the model's next token distribution.

\begin{figure}[t]
\centering
\resizebox{0.98\linewidth}{!}{%
\input{tikz/router_training_method}
}
\caption{
    \looseness=-1
    \textbf{KV-Compression Aware Training (\ourmethod).} A context $\va$ goes through both masked (left) and dense (right) forward passes. In the masked forward pass, routers (orange) compute masks for groups of consecutive layers, marking KV slots as active (green) or inactive (muted green). In the dense forward pass (blue), all KV slots are kept. The output of the masked forward pass is used to compute $\gL_\mathrm{mask}$, router distributions are used to compute $\gL_\mathrm{budget}$, and the outputs of the dense forward pass are used to compute $\gL_\mathrm{anchor}$. These are jointly used to update the parameters.
}
\label{fig:router-training-method}
\vspace{-10pt}
\end{figure}
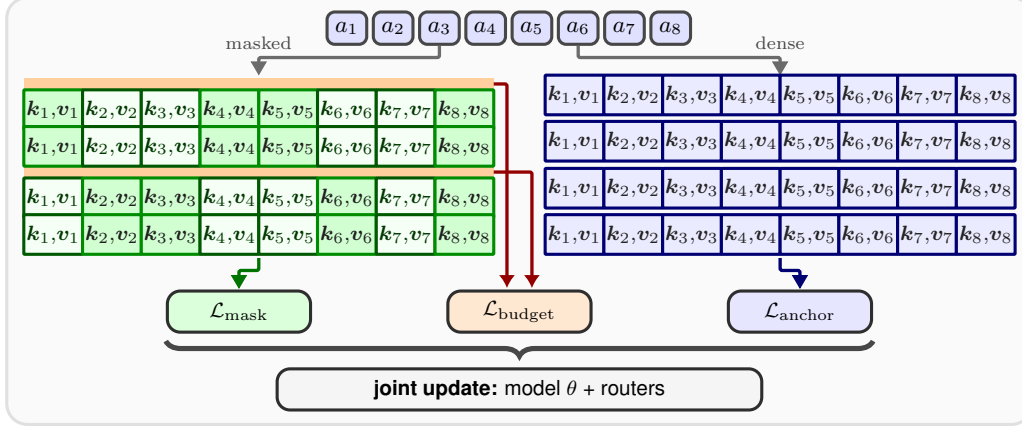
Our theoretical results show that for almost any sequence-to-vector function, there exist transformer implementations whose prefix can be compressed to a single KV pair, as well as implementations for which any non-trivial compression incurs a constant error. To build intuition, we consider a motivating example: character histogram computation. We show that natural transformer implementations of this computation may produce incompressible token representations, whereas more structured alternatives are highly compressible. This motivates treating KV compressibility as an explicit training objective.

\looseness=-1
Guided by this perspective, we propose \textbf{KV}-\textbf{C}ompression \textbf{A}ware \textbf{T}raining (\ourmethod), a continued pretraining (CPT) procedure that promotes the emergence of KV-compressible internal representations. Starting from a pretrained transformer, \ourmethod introduces a train-time KV sparsification policy which masks out a constant fraction of the KV slots. The training objective combines a self-distillation loss, which matches the masked model's distribution to the dense model's distribution, and an NTP loss applied to the unmasked forward pass to preserve uncompressed model behavior. This exposes the model to the information bottleneck induced by KV cache compression during training, encouraging it to reorganize its representations into a more compressible form while maintaining performance in the uncompressed setting.

\looseness=-1
To demonstrate that \ourmethod improves the quality–budget tradeoff of post-hoc KV cache compression, we apply it to \textsc{Qwen2.5} models~\citep{qwen2.5}, and evaluate state-of-the-art optimization-based compression methods on the KV caches of the resulting checkpoints. We measure performance along three axes: (i) suffix perplexity under prefix compression, (ii) retrieval accuracy from a compressed prefix, and (iii) compressed long-context question answering on LongBench v2~\citep{bai2025longbench} tasks. Across a range of compression budgets, model sizes, and compression methods, models trained with \ourmethod consistently achieve a better quality–budget tradeoff than the base model, yielding improvements of up to \textbf{3.21}$\times$ in suffix perplexity retention, \textbf{5}$\times$ in optimization time, \textbf{68\%} in retrieval accuracy, and \textbf{39\%} in long-context QA. Importantly, our goal is not to replace post-hoc methods, but to make models more amenable to them, enabling more effective compression.

\textbf{Contributions.} In summary, this paper makes the following contributions:
\begin{enumerate}[label=(\arabic*)]
    \item We characterize KV-compressibility as a property of the learned transformer representation, rather than solely of the task or the input sequence.
    \item We provide theoretical results showing that almost all sequence-to-vector functions can admit both compressible and non-compressible transformer implementations.
    \item We propose \ourmethod, a training procedure that promotes compressible KV representations. 
    \looseness=-1
    \item We show that \ourmethod improves post-hoc KV cache compression at matched cache and optimization budgets across: suffix perplexity retention, retrieval, and long-context QA.
\end{enumerate}


%% file: tikz/router_training_method.tex
\begin{tikzpicture}[
    >=Latex,
    font=\sffamily\footnotesize,
    token/.style={draw=black!72, very thick, rounded corners=3pt, minimum width=0.54cm, minimum height=0.44cm, inner sep=1.2pt, align=center, font=\sffamily\small},
    slot/.style={draw=black!72, very thick, rounded corners=0.4pt, minimum width=0.80cm, minimum height=0.54cm, inner sep=0.5pt, align=center, font=\sffamily\footnotesize},
    routerbar/.style={draw=none, fill=orange!38, minimum width=6.40cm, minimum height=0.16cm, inner sep=0pt, outer sep=0pt},
    loss/.style={draw=black!82, very thick, rounded corners=5pt, minimum width=1.95cm, minimum height=0.56cm, align=center, inner sep=3pt},
    update/.style={draw=black!82, very thick, rounded corners=5pt, fill=black!4, minimum width=6.60cm, minimum height=0.56cm, align=center, inner sep=4pt},
    active/.style={fill=green!18, draw=green!55!black, text=black!88},
    masked/.style={fill=green!4, draw=green!35!black, text=green!22!black},
    dense/.style={fill=blue!8, draw=blue!45!black, text=black!86},
    flowline/.style={line width=1.28pt, draw=black!58, rounded corners=0.8pt},
    routerline/.style={line width=1.28pt, draw=red!58!black, rounded corners=0.8pt},
    maskline/.style={line width=1.28pt, draw=green!45!black, rounded corners=0.8pt},
    denseline/.style={line width=1.28pt, draw=blue!45!black, rounded corners=0.8pt},
    forwardlabel/.style={above=3.0pt, font=\sffamily\scriptsize, text=black!72, fill=black!2, inner xsep=1pt, inner ysep=0.1pt}
]

\foreach \i/\x in {1/5.25,2/5.88,3/6.51,4/7.14,5/7.77,6/8.40,7/9.03,8/9.66} {
    \node[token, fill=blue!12] (tok\i) at (\x,4.58) {$a_\i$};
}

\draw[flowline] (tok3.south) -- (6.51,4.18) -- (4.05,4.18) -- (4.05,4.10);
\draw[flowline] (tok6.south) -- (8.40,4.18) -- (11.15,4.18) -- (11.15,4.10);
\node[forwardlabel] at (4.05,4.18) {$\mathrm{masked}$};
\node[forwardlabel] at (11.15,4.18) {$\mathrm{dense}$};
\path[fill=black!58, draw=black!58] (4.05,3.95) -- ++(-0.10,0.17) -- ++(0.20,0) -- cycle;
\path[fill=black!58, draw=black!58] (11.15,3.95) -- ++(-0.10,0.17) -- ++(0.20,0) -- cycle;

\node[routerbar] (routerA) at (4.05,3.82) {};
\node[routerbar] (routerB) at (4.05,2.62) {};

\foreach \i/\x/\sty in {1/1.25/active,2/2.05/masked,3/2.85/masked,4/3.65/active,5/4.45/active,6/5.25/masked,7/6.05/masked,8/6.85/active} {
    \node[slot,\sty] (rA\i) at (\x,3.47) {$\vk_\i{,}\vv_\i$};
}

\foreach \i/\x/\sty in {1/1.25/active,2/2.05/masked,3/2.85/masked,4/3.65/active,5/4.45/active,6/5.25/masked,7/6.05/masked,8/6.85/active} {
    \node[slot,\sty] (rB\i) at (\x,2.97) {$\vk_\i{,}\vv_\i$};
}

\foreach \i/\x/\sty in {1/1.25/masked,2/2.05/active,3/2.85/active,4/3.65/masked,5/4.45/masked,6/5.25/active,7/6.05/masked,8/6.85/active} {
    \node[slot,\sty] (rC\i) at (\x,2.27) {$\vk_\i{,}\vv_\i$};
}
\foreach \i/\x/\sty in {1/1.25/masked,2/2.05/active,3/2.85/active,4/3.65/masked,5/4.45/masked,6/5.25/active,7/6.05/masked,8/6.85/active} {
    \node[slot,\sty] (rD\i) at (\x,1.77) {$\vk_\i{,}\vv_\i$};
}

\foreach \i/\x in {1/8.35,2/9.15,3/9.95,4/10.75,5/11.55,6/12.35,7/13.15,8/13.95} {
    \node[slot,dense] (dA\i) at (\x,3.67) {$\vk_\i{,}\vv_\i$};
    \node[slot,dense] (dB\i) at (\x,3.04) {$\vk_\i{,}\vv_\i$};
    \node[slot,dense] (dC\i) at (\x,2.40) {$\vk_\i{,}\vv_\i$};
    \node[slot,dense] (dD\i) at (\x,1.77) {$\vk_\i{,}\vv_\i$};
}

\node[loss, fill=green!14] (mask) at (3.78,0.72) {$\mathcal{L}_{\mathrm{mask}}$};
\node[loss, fill=orange!18] (budget) at (7.60,0.72) {$\mathcal{L}_{\mathrm{budget}}$};
\node[loss, fill=blue!10] (anchor) at (11.42,0.72) {$\mathcal{L}_{\mathrm{anchor}}$};

\draw[routerline, -{Latex[length=5.0pt,width=6.4pt]}] (routerA.east) -- (7.43,3.82) -- ($(budget.north)+(-0.17,0.02)$);
\draw[routerline, -{Latex[length=5.0pt,width=6.4pt]}] (routerB.east) -- (7.77,2.62) -- ($(budget.north)+(0.17,0.02)$);
\draw[maskline] (4.05,1.45) -- (4.05,1.28) -- (3.78,1.28) -- ($(mask.north)+(0,0.17)$);
\draw[denseline] (11.15,1.45) -- (11.15,1.28) -- (11.42,1.28) -- ($(anchor.north)+(0,0.17)$);
\path[fill=green!45!black, draw=green!45!black] ($(mask.north)+(0,0.02)$) -- ++(-0.10,0.17) -- ++(0.20,0) -- cycle;
\path[fill=blue!45!black, draw=blue!45!black] ($(anchor.north)+(0,0.02)$) -- ++(-0.10,0.17) -- ++(0.20,0) -- cycle;

\draw[decorate, decoration={brace, mirror, amplitude=5.8pt}, line width=2.0pt, draw=black!72]
  ($(mask.south west)+(0,-0.10)$) -- ($(anchor.south east)+(0,-0.10)$);
\node[update] (jointUpdate) at (7.60,-0.34) {\textbf{joint update:} model $\theta$ + routers};

\begin{scope}[on background layer]
  \node[draw=black!12, very thick, rounded corners=8pt, fill=black!2,
        fit=(tok1) (tok8)
            (routerA) (routerB)
            (rA1) (rA8) (rD1) (rD8) (dA1) (dA8) (dD1) (dD8)
            (budget) (mask) (anchor) (jointUpdate),
        inner xsep=6pt, inner ysep=5pt] {};
\end{scope}
\end{tikzpicture}%

%% file: sections/preliminaries.tex
\section{KV Cache Compression: Problem Formulation}
In this section, we theoretically formalize the KV cache compression problem. A \emph{KV cache compression policy} for a transformer with $L$ layers is a collection of functions $\compress = (\cl_1,\dots,\cl_L),$ where each $\cl_\ell$ maps sequences of $n$ key--value pairs to sequences of length $\length(n) \le n$. For KV pairs in layer $\ell$
\begin{equation}
    (\mK,\mV) = ([\vk_1,\dots,\vk_n]^\top, [\vv_1,\dots,\vv_n]^\top),
\end{equation}
we write
\begin{equation}
    \cl_\ell(\mK,\mV) = ([\tilde \vk_1,\dots,\tilde \vk_{\length(n)}]^\top, [\tilde \vv_1,\dots,\tilde \vv_{\length(n)}]^\top),
\end{equation}

\looseness=-1
where $\length(n)$ is referred to as the \emph{compression budget}. Given a transformer $\model$, a compression policy $\compress$, and a sequence of tokens $\va = (a_1, \dots, a_n)$, we define a compressed model $\model_{\compress, \va}$ that shares the parameters of $\model$ but uses a modified forward pass. For an input sequence $\vb$, the computation of $\model_{\compress, \va}(\vb)$ proceeds as in $\model([\va,\vb])$, except for the attention over prefix tokens, which is augmented the following way: at layer $\ell$, let $(\mK_\va,\mV_\va)$ denote the KV cache of $\va$, based on the original model $\model$ and let $\mY$ be the representations of $\vb$ after layer $\ell-1$ of $\model_{\compress, \va}$. We compute
\begin{equation}
    \mQ_\vb = \mY \mW_Q, \quad \mK_\vb = \mY \mW_K, \quad \mV_\vb = \mY \mW_V,
\end{equation}
and replace the prefix KV cache $(\mK_\va,\mV_\va)$ with its compressed version $(\tilde \mK_\va, \tilde \mV_\va) = \cl_\ell(\mK_\va,\mV_\va)$. The attention computation then becomes
\begin{equation}
    \mathrm{AttnHead}(\mY)
=
\mathrm{softmax}\!\left(
\frac{1}{\sqrt{d_k}}\mQ_\vb
\begin{bmatrix}
\tilde \mK_\va \\
\mK_\vb
\end{bmatrix}^{\!\top}
\right)
\begin{bmatrix}
\tilde \mV_\va \\
\mV_\vb
\end{bmatrix}.
\end{equation}

\looseness=-1
Finally, following standard convention in next-token prediction transformers, we define the output of $\model_{\compress, \va}(\vb)$ to be the representation of the final token.
Thus, compression modifies only the prefix KV pairs, leaving the remainder of the computation unchanged. We now formalize the notion of transformer \emph{KV-compressibility}.
\begin{definition}\label{def:compressibility}
Fix $N \in \sN$, $\varepsilon > 0$, and let $\length : \sN \to \sN$ be a budget function. A transformer $\model$ is said to be $(N, \varepsilon, \length)$-compressible\footnote{We include $N$ explicitly, as many natural sequence-to-vector functions require model dimension scaling with $N$ (e.g., $O(N)$ or $O(\log N)$) in order to achieve arbitrarily accurate approximation; see e.g.,  \citet{sanford2023representational, yehudai2024can}.} if there exists a KV cache compression policy $\compress$ with budget $\length$ such that for every pair of sequences $\va$, $\vb$ of lengths $n$ and $k$ respectively, and with a combined length $n + k \le N$, it holds that
\begin{equation}
\| \model([\va, \vb]) - \model_{\compress, \va}(\vb) \| < \varepsilon.
\end{equation}
\end{definition}

A more formal treatment of these terms and definitions is provided in Appendix~\ref{app:preliminaries}.

%% file: sections/motivation_and_theoretical_results.tex
\section{Motivation and Theoretical Results}

\looseness=-1
We begin with a theoretical analysis of transformer KV-compressibility, motivated by the following central question:
\vspace{-2mm}
\begin{insight}{Key question}
\emph{Can the same sequence-to-vector function admit transformer implementations with different levels of KV-compressibility?}
\end{insight}
The following theorem shows that for almost any sequence-to-vector function, there exist transformer architectures that admit both fully compressible and inherently non-compressible implementations. The proof is deferred to Appendix~\ref{app:main_theorem}.

\begin{theorem}
\label{thm:main_theorem}
Let $\alb$ be a finite alphabet, and let $\func : \bigcup_{n \le N} \alb^n \to \mathbb{R}^{d_{\text{out}}}$ be a sequence-to-vector function. Suppose that there exists a sequence $\va$ of length $n$ and two sequences $\vb^1, \vb^2$ of length $k \le N-n$ such that
\begin{equation}
    \func([\va, \vb^1]) \neq \func([\va, \vb^2]).
\end{equation}
 Then for every $\varepsilon, \const > 0$:
\begin{enumerate}
\item There exists a transformer that approximates $\func$ and is $(N, \varepsilon, 1)$-compressible.
\item There exists a transformer of the same architecture, that approximates $\func$ but is not $(N, \const, \length)$-compressible for any budget function satisfying $\length(n)<n$.
\end{enumerate}
\end{theorem}

We note that, as KV cache compression methods perform well in practice, the worst-case behavior implied by this theorem may not typically arise; understanding in what scenarios this occurs is an interesting direction for future work. Nevertheless, the result highlights that compressibility can vary across transformer implementations, suggesting the potential benefit of guiding transformers toward more compressible representations at train time. Before turning to training methods, we build intuition through a simple motivating example. We show that even for natural sequence-to-vector functions, simple transformer implementations can yield non-compressible representations, while more structured, highly compressible implementations exist.



\textbf{Motivating example: histogram computation.} Let $A = [m]$ be a finite alphabet. Given a sequence $\va = (a_1,\dots,a_n) \in A^n$, the histogram function computes the empirical distribution of symbols,
\begin{equation}
    \hist(\va) = \left(\tfrac{n_1}{n}, \dots, \tfrac{n_m}{n}\right),
\end{equation}
where $n_i = |\{j : a_j = i\}|$. This is a natural sequence-to-vector function that depends only on aggregate statistics of the input. A straightforward way to implement this function with a 2-layer transformer is as follows. First, take the embedding map to be\footnote{We typically assume token embeddings of the form $\te(a_i, i) = \vu_{a_i}+\vp_i$. The embedding map used above can be obtained from this representation by a minor adjustment to the first attention layer; see Appendix~\ref{app:motivating_example} for details.} $\te(a_i, i) = \ve_{a_i}$. We take the first transformer layer to be the identity \footnote{This can be done by setting $\mW_O = \mathbf{0}$ and choosing the feedforward network to be the identity.},
and then apply a second attention layer with uniform attention weights (by setting $\mW_Q = \mathbf{0}$),  So that the representation of the final token becomes the average of all previous token embeddings, obtaining
\begin{equation}
    \model(\va) = \frac{1}{n} \sum_{i=1}^n \ve_{a_i} = \hist(\va).
\end{equation}
This gives a natural implementation of the histogram function with a very simple architecture (see Appendix~\ref{app:motivating_example} for details). Interestingly, we find that this implementation is not compressible:
\begin{proposition}
\label{prop:hist_non_compressible}
\looseness=-1
For any $N \in \sN$ and any budget function $\length$ with $\length(N-1) < N-1$ there exists a constant $\const>0$ such that the above transformer implementation of $\hist$ is not $(N, \const, \length)$-compressible.
\end{proposition}
\begin{proof}[Proof sketch.]
Recall that in the compressed setting the input is $[\va, \vb]$ and we want to compress the KV caches associated with the prefix $\va$. We begin by observing that since in the first layer $\mW_O =\mathbf{0}$, compressing the prefix of the first layer has no effect on the representation of the tokens of $\vb$. Additionally, since in the second layer $\mW_Q = \mathbf{0}$, the attention scores are independent of the keys, and therefore the attention weights are uniform across all tokens. Thus, for a prefix sequence $\va$ of length $n$ and a suffix $\vb$ of length $k$, the compressed model satisfies, for any compression policy $\compress$,
\begin{equation}
    \model_{\compress, \va}(\vb) = \frac{1}{\length(n)+k} \Big( \sum_{i=1}^{\length(n)} \tilde \vv_i + \sum_{i=1}^k \ve_{b_i} \Big).
\end{equation}
Similarly, the full model computes
\begin{equation}
    \model([\va,\vb]) = \frac{1}{n+k} \Big( \sum_{i=1}^{n} \ve_{a_i} + \sum_{i=1}^k \ve_{b_i} \Big).
\end{equation}
For notational convenience, define
\begin{equation}
\tilde \vv = \frac{1}{\length(n)+k} \sum_{i=1}^{\length(n)} \tilde \vv_i, \quad
\tilde \va = \frac{1}{n+k} \sum_{i=1}^{n} \ve_{a_i}, \quad
\tilde \vb =  \Big(\frac{1}{\length(n)+k} - \frac{1}{n+k}\Big) \sum_{i=1}^{k} \ve_{b_i}.
\end{equation}

Then the difference between the compressed and full outputs can be written as
\begin{equation}
\label{eq:compress_cond_sketch}
    \| \model_{\compress, \va}(\vb) - \model([\va,\vb]) \|
    =
    \left\|
    \tilde \vv - \tilde \va
    +\tilde \vb
    \right\|.
\end{equation}
Both $\tilde \vv$ and $\tilde \va$ depend only on the prefix $\va$, while $\tilde \vb$ depends on the suffix $\vb$ and can vary over a range of values. For $\model_{\compress, \va}(\vb)$ to $\varepsilon$-approximate $\model([\va,\vb])$ uniformly over all suffixes $\vb$, Equation~\ref{eq:compress_cond_sketch} requires that the fixed term $\tilde \vv - \tilde \va$ be $\varepsilon$-close to every possible value of $\tilde \vb$. For sufficiently small $\varepsilon$, this is impossible. Hence, the error is bounded below by a constant.

\end{proof}

For a formal proof of Proposition \ref{prop:hist_non_compressible}, see Appendix \ref{app:motivating_example}. We now show that a slight modification of the above implementation yields highly compressible KV representations.

\begin{proposition}
\label{prop:hist_compressible}
For every $N \in \sN$ and $\varepsilon > 0$, there exists a $\hist$ transformer implementation  with the same architecture as above that is $(N,\varepsilon,\length)$-compressible with $\length(n) = 1$.
\end{proposition}

\looseness=-1
\begin{proof}[Proof sketch]
We assume an injective token embedding of the form $\te(a_i, i) = \vu_{a_i}+\vp_i$ where $\vp_i$ are positional embeddings. As in the previous construction, we set $\mW_O = \mathbf{0}$ in the first transformer layer, but now choose the feedforward network to approximately satisfy
\begin{equation}
\rho_1(\vu_{a_i} + \vp_i) =[\ve_{a_i}^\top, \vp_i^\top, \mathbf{0}^\top, \mathbf{0}^\top]^\top
\end{equation}
That is, each token representation is partitioned into four blocks: the first encodes token identity, the second retains positional information, and the last two are auxiliary slots used later in the construction. We then apply a second attention layer with $\mW_Q = \mathbf{0}$, so attention is uniform. With the residual connection, the final token representation can be made to contain the empirical average in the third block while preserving the positional encoding $\vp_n$, which encodes the sequence length.
The final feedforward network $\rho_2$ is designed to behave differently depending on the available information. In the uncompressed setting, the value of the final block remains $\mathbf{0}$, in which case it simply returns the empirical average (from the third slot). In the compressed setting, however, an optimal compression policy can provide additional positional information in the final vector block, which can then be used to correct for the distortion introduced by compression.
Specifically, we define a compression policy that maps the prefix to a single KV pair, whose value stores the \emph{unnormalized} histogram together with the prefix length (encoded via $\vp_n$ in the last block). When processing a suffix, attention averages this summary with the suffix tokens, producing a mis-scaled estimate of the histogram. Crucially, the model now has access to both the original prefix length and the total sequence length (via positional encodings). The final feedforward network uses this information to re-normalize the average and recover the correct histogram. This establishes $(N,\varepsilon,\length)$-compressibility with $\length(n) = 1$.

\end{proof}

For a formal proof of Proposition \ref{prop:hist_compressible}, see Appendix \ref{app:motivating_example}. Note that while incorporating positional information to correct attention weights under compression is conceptually straightforward, a standard transformer trained only on full (uncompressed) sequences is not explicitly encouraged to learn such a mechanism. Since computing the histogram depends only on token identities and is invariant to their positions, the model can, in principle, learn to ignore positional information. This suggests that obtaining compressible solutions may benefit from a training procedure that encourages compression.


%% file: sections/method.tex
\section{Training for KV Cache Compressibility}
\begin{table*}[t]
  \centering
  \small
  \caption{\textbf{\ourmethod matches base model performance.} We test our \ourmethod-trained \textsc{Qwen2.5} checkpoints on a variety of QA benchmarks without any KV cache compression applied. Values are normalized multiple-choice accuracy (\%) with 1000 examples per task.}\label{tab:fineweb-reasoning}
  \vspace{10pt}
  \resizebox{\linewidth}{!}{
      \begin{tabular}{llcccccccc}
      \toprule
      Model & Variant & HellaSwag & WinoGrande & PIQA & OpenBookQA & ARC-E &
      ARC-C & Avg. \\
      \midrule
      \multirow{2}{*}{\textsc{Qwen2.5-0.5B}}
        & Base & \textbf{54.6} & 52.9 & 70.4 & \textbf{38.2} & 60.8
      & 32.2 & 51.5  \\
        & \ourmethod & 53.6 & \textbf{54.1} & \textbf{72.1} & 37.4 &
      \textbf{63.6} & \textbf{32.4} & \textbf{52.2}  \\
      \midrule
      \multirow{2}{*}{\textsc{Qwen2.5-1.5B}}
        & Base & \textbf{68.6} & \textbf{60.5} & \textbf{76.2} &
      40.0 & 73.3 & \textbf{45.2} & \textbf{60.6}  \\
        & \ourmethod & 66.4 & \textbf{60.5} & 75.5 & \textbf{40.6} &
      \textbf{73.9} & 43.5 & 60.1 \\
      \bottomrule
      \end{tabular}
  }
  \vspace{-3mm}
\end{table*}
As shown in the previous section, training transformers to solve sequence-to-vector tasks may admit both compressible and non-compressible solutions.
This naturally raises the following question:
\vspace{-2mm}
\begin{insight}
 cCan we guide transformers toward more compressible representations during training? 
\end{insight}
\vspace{-1mm}
To answer this question, we introduce KV-Compression Aware Training (\ourmethod), a continued pretraining procedure that explicitly encourages compressible internal representations.

\looseness=-1
\textbf{Train-time KV sparsification policy.}
Our goal is to train transformers whose representations are more amenable to SOTA optimization-based KV cache compression methods. Ideally, one would incorporate such methods directly into the training loop, adapting the compression policy as the model evolves; however, this is computationally prohibitive. A practical alternative is to introduce a simple and scalable training-time KV sparsification policy, such as random KV slot dropping, attention-based filtering, or a lightweight parameterized policy that adapts jointly with the model. In our ablations (Appendix~\ref{app:method-ablation}), we find that a parameterized policy provides the strongest performance, in terms of both downstream compression and retention of uncompressed accuracy. We hypothesize that this is because adaptive policies better align with the test-time behavior of optimization-based compressors, whereas fixed policies may be too rigid. Accordingly, in our experiments, we implement \ourmethod using lightweight learned routers, described below, while emphasizing that the framework is general and can accommodate a range of sparsification policies.

\textbf{Routing mechanism.}
Starting from a pretrained decoder-only transformer, we insert learned \emph{routers} between consecutive layers and train them jointly with the model. At layer $\ell$, the router takes as input the token representations from layer $\ell-1$ and outputs a scalar score in $[0,1]$ for each token, indicating its importance. These scores are thresholded using a (non-trainable) hyperparameter $\tau$ to produce a binary mask: tokens with scores above $\tau$ are active in attention, while the rest are masked out. For efficiency, routers are implemented as linear attention modules. Routers are initialized with all tokens active (i.e., all scores are set to $1$), so training begins from the standard dense transformer and gradually learns to mask tokens, using a budget loss (defined below). To further reduce overhead, routers are shared across groups of consecutive layers. See Appendix~\ref{app:method} for additional details.

\textbf{Training objective.}
We train the transformer using an augmented next-token prediction (NTP) objective.
For a sequence $\va=(a_1,\dots,a_n)$, let $p_\theta^{\mathrm{mask}}(\cdot \mid \va_{<i})$ denote the model distribution when masking is applied to the KV cache, and $p_\theta^{\mathrm{dense}}(\cdot \mid \va_{<i})$ denote the distribution under the standard (unmasked) forward pass. During training, we optimize
\begin{equation} \label{eq:training_objective}
    \mathcal{L}(\theta) = \lambda_\mathrm{mask}\mathcal{L}_\mathrm{mask} + \lambda_{\mathrm{budget}}\mathcal{L}_{\mathrm{budget}} + \lambda_\mathrm{anchor}\mathcal{L}_\mathrm{anchor}.
\end{equation}
The first term is given by
\begin{equation}
    \mathcal{L}_{\mathrm{mask}} =
    \frac{1}{n}\sum_{i=1}^{n} \KL\!\left(\operatorname{sg}[p_\theta^{\mathrm{dense}}(\cdot \mid a_{<i})] \,\middle\|\, p_\theta^{\mathrm{mask}}(\cdot \mid a_{<i})\right),
\end{equation}
where $\operatorname{sg}$ denotes the stop-gradient operation. This term trains the masked forward pass to match the dense model distribution via self-distillation. When the KV-sparsification policy is learnable, we include a budget loss designed to maintain a desirable retention rate.  Following~\citet{hwang2025hnet}, in our implementation, we use:
\begin{equation}
\mathcal{L}_\mathrm{budget}
= \rho^{-1} FG + (1-\rho)^{-1}(1 - F)(1 - G),
\quad F = \frac{1}{L} \sum_{i=1}^{L} m_i,
\quad G = \frac{1}{L} \sum_{i=1}^{L} q_i,
\end{equation}
where $q_i$ and $m_i$ denote the router’s score and binary mask for token $i$, respectively, and $\rho$ is the target retention rate. The loss is averaged across router layers, and gradients are taken with respect to the scores $q_i$ only. The final term
\begin{equation}
    \mathcal{L}_{\mathrm{anchor}} =\frac{1}{n}\sum_{i=1}^{n} -\log p_\theta^{\mathrm{dense}}(a_i \mid a_{<i}),
\end{equation}
applies standard NTP to the output of the dense forward pass, preserving its distribution and providing a stable teacher for distillation.\footnote{We note that a natural alternative to our training objective is to use distillation from a frozen base model for both the masked and dense forward passes. We do not use this variant since it requires keeping an additional frozen teacher model in memory and running a third forward pass, significantly increasing training cost.} See Figure~\ref{fig:router-training-method} for an illustration of this training procedure.

\textbf{Inference.}
At evaluation time, we use the unmasked forward pass, on top of which we can apply standard KV cache compression methods. Thus, the output of \ourmethod is a standard transformer whose representations are trained to be more amenable to compression.

%% file: sections/experiments.tex
\section{Empirical Evaluation}\label{sec:experiments}
\begin{figure}[t]
\vspace{-3mm}
    \centering
    \includegraphics[width=\linewidth]{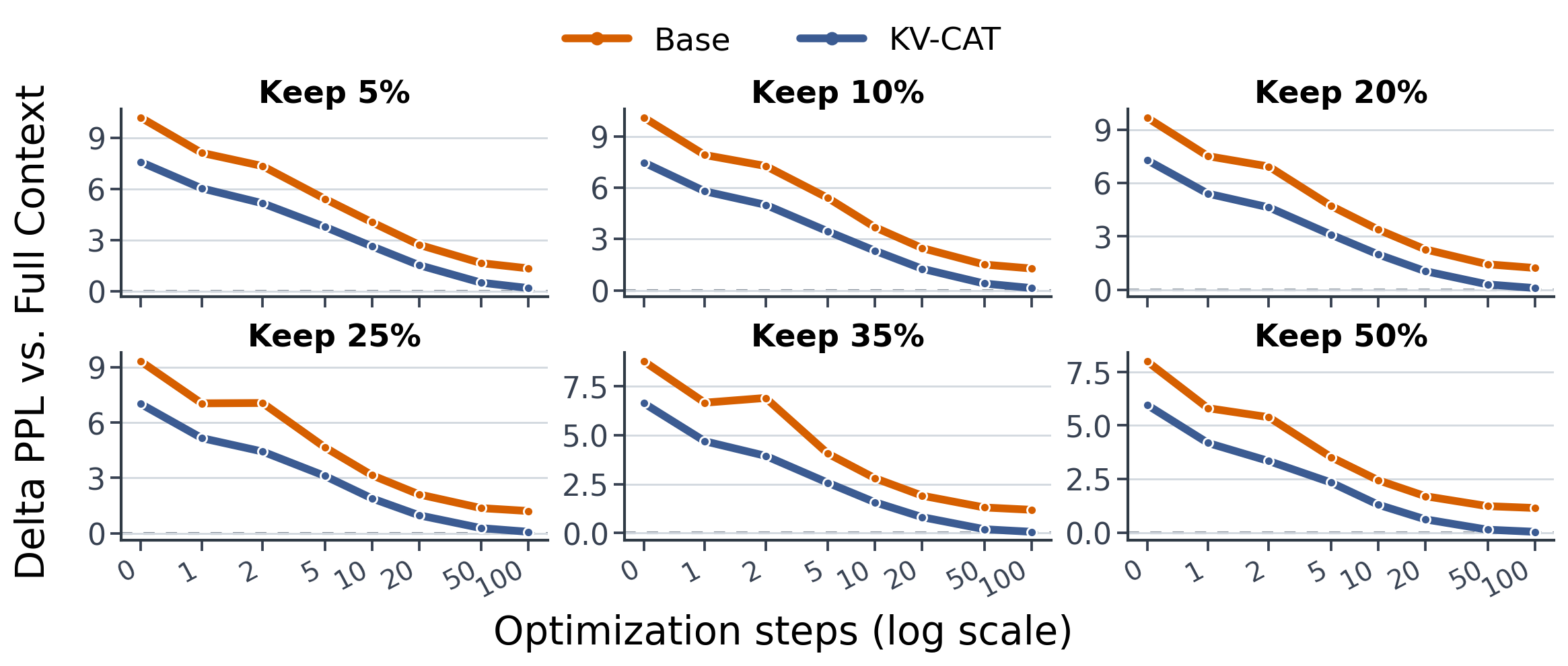}
    \vspace{-10pt}
    \caption{\textbf{\ourmethod speeds up gradient-based KV cache compression.} We plot the gap in suffix perplexity under full/compressed-prefix inference throughout gradient-based KV cache optimization. Each panel fixes a different KV keep ratio. Across ratios, the \ourmethod checkpoint achieves a comparable $\Delta$PPL in fewer optimization steps than the base model, yielding up to a \textbf{5}$\times$ speedup.}\label{fig:qwen0p5_ppl_plots_main}
    \vspace{-4mm}
\end{figure}
\begin{wraptable}[16]{r}{0.44\textwidth}
  \vspace{-20pt}
  \centering
  \caption{\textbf{Needle retrieval from compressed haystack.} Exact-match accuracy (\%) over 100 examples per keep ratio.}\label{tab:niah}
  \vspace{4pt}
  \resizebox{\linewidth}{!}{
      \begin{tabular}{@{}ccccc@{}}
      \toprule
      & \multicolumn{2}{c}{\textsc{Qwen2.5-0.5B}} &
        \multicolumn{2}{c}{\textsc{Qwen2.5-1.5B}} \\
      \cmidrule(lr){2-3}\cmidrule(l){4-5}
      Keep & Base & \ourmethod & Base & \ourmethod \\
      \midrule
      5\% & \textbf{18} & 17 & \textbf{20} & \textbf{20} \\
      10\% & 12 & \textbf{15} & \textbf{20} & \textbf{20} \\
      15\% & 15 & \textbf{16} & \textbf{21} & \textbf{21} \\
      20\% & \textbf{15} & 13 & 22 & \textbf{23} \\
      25\%& 20 & \textbf{22} & 24 & \textbf{28} \\
      30\% & 23 & \textbf{34} & 41 & \textbf{44} \\
      35\% & 21 & \textbf{32} & 46 & \textbf{54} \\
      40\% & 24 & \textbf{38} & 42 & \textbf{55} \\
      50\% & 28 & \textbf{47} & 49 & \textbf{67} \\
      \midrule
      Mean & 19.6 & \textbf{26.0} & 31.7 & \textbf{36.9} \\
      \bottomrule
      \end{tabular}
  }
  \vspace{-0.0em}
\end{wraptable}
\looseness=-1
We evaluate the effect of \ourmethod on the performance of downstream optimization-based KV cache compression techniques. Our empirical study is guided by four central questions: \textbf{(Q1)} Does \ourmethod training preserve base model performance in uncompressed settings? \textbf{(Q2)} Does \ourmethod improve compression quality under a fixed KV-optimization budget? \textbf{(Q3)} Does \ourmethod improve retrieval from a KV-compressed context? \textbf{(Q4)} Do these gains transfer to KV-compressed long-context question answering? Across all experiments, we compare the original pretrained model to the model obtained via \ourmethod continued pretraining, applying the same downstream compression method at matched KV retention and optimization budgets. We provide full experimental details in Appendix~\ref{app:exp_details}, and additional results in Appendix~\ref{app:add_exp}.

\textbf{\ourmethod training setup.}
We apply \ourmethod to \textsc{Qwen2.5-0.5B} and \textsc{Qwen2.5-1.5B} checkpoints via continued pretraining on FineWeb-Edu~\citep{penedo2024fineweb}. In both cases, we train all model parameters, and introduce four learned routers shared across four layer groups. 
Each model is trained on a total of $5.24\times10^9$ tokens with a max learning rate of $10^{-4}$. See further details in Appendix~\ref{app:cpt}.

\begin{wraptable}[22]{r}{0.52\linewidth}
      \centering
      \vspace{-20pt}
      \caption{
      \looseness=-1
      \textbf{Suffix completion under prefix compression.} Each example uses a 768-token prefix and a 256-token suffix. Prefix KVs are compressed using Attention Matching. We report: (1) $\Delta$PPL, the difference in suffix perplexity under the full prefix vs. compressed prefix; (2) KL, the divergence between the two token distributions; and (3) Top-1, percentage of tokens on
  which the two distributions' top-1 agree.}\label{tab:attention-matching-ppl}
      \vspace{4pt}
      \resizebox{\linewidth}{!}{
          \begin{tabular}{@{}ccc|ccc@{}}
          \hline
          Model & Keep & Variant & $\Delta$PPL ($\downarrow$) & KL ($\downarrow$) & Top-1 ($\uparrow$) \\
          \hline
          \multirow{8}{*}{\rotatebox[origin=c]{90}{\textsc{Qwen2.5-0.5B}}}
          & \multirow{2}{*}{5\%} & Base & 0.780 & 0.0562 & 88.5\% \\
          &                       & \ourmethod & \textbf{0.461} & \textbf{0.0385} & \textbf{91.1\%} \\
          & \multirow{2}{*}{10\%} & Base & 0.662 & 0.0483 & 89.3\% \\
          &                        & \ourmethod & \textbf{0.422} & \textbf{0.0321} & \textbf{91.4\%} \\
          & \multirow{2}{*}{20\%} & Base & 0.777 & 0.0549 & 88.7\% \\
          &                        & \ourmethod & \textbf{0.438} & \textbf{0.0374} & \textbf{90.8\%} \\
          & \multirow{2}{*}{40\%} & Base & 0.592 & 0.0431 & 90.3\% \\
          &                        & \ourmethod & \textbf{0.360} & \textbf{0.0272} & \textbf{92.5\%} \\
          \hline
          \multirow{8}{*}{\rotatebox[origin=c]{90}{\textsc{Qwen2.5-1.5B}}}
          & \multirow{2}{*}{5\%} & Base & 0.573 & 0.0570 & 89.2\% \\
          &                       & \ourmethod & \textbf{0.292} & \textbf{0.0361} & \textbf{90.9\%} \\
          & \multirow{2}{*}{10\%} & Base & 1.027 & 0.0796 & 88.3\% \\
          &                        & \ourmethod & \textbf{0.320} & \textbf{0.0366} & \textbf{90.7\%} \\
          & \multirow{2}{*}{20\%} & Base & 1.093 & 0.0876 & 87.7\% \\
          &                        & \ourmethod & \textbf{0.536} & \textbf{0.0491} & \textbf{90.2\%} \\
          & \multirow{2}{*}{40\%} & Base & 0.742 & 0.0653 & 90.0\% \\
          &                        & \ourmethod & \textbf{0.259} & \textbf{0.0326} & \textbf{92.4\%} \\
          \hline
          \end{tabular}
      }
  \end{wraptable}
\textbf{KV cache compression methods.}
\looseness=-1
Across experiments, we use two post-hoc KV cache compression procedures: Attention Matching~\citep{zweiger2026fastkv} and a gradient-based KV cache optimization method adapted from~\citet{eyuboglu2025cartridges}. Attention Matching constructs a compact prefix cache layer-by-layer, approximating the dense model's attention traces. The gradient-based method instead directly optimizes the compact prefix KV cache to match the dense model's logits on suffix tokens. Both methods compress the KV cache of a given prefix \texttt{<text>} by leveraging supervision from a suffix, which together form a \emph{reconstruction sequence}. For the retrieval and QA evaluations,
we follow \citet{zweiger2026fastkv}, and use a reconstruction sequence of the form \texttt{<text><instruction><text>}, where \texttt{<instruction>} is set to ``\texttt{Reproduce the preceding passage verbatim.}'', and the loss is applied only on the second repeated \texttt{<text>} span. Additional implementation details are given in Appendix~\ref{app:exp_details}.

\textbf{Preserving uncompressed performance (Q1).}
We evaluate the effect of \ourmethod on the performance of the model \emph{with no KV cache compression applied}. We compare the base model and the \ourmethod-trained checkpoint on six standard multiple-choice benchmarks using the LightEval harness~\citep{habib2023lighteval}. As shown in Table~\ref{tab:fineweb-reasoning},
\ourmethod closely matches the base model, with an average gain of 0.7 accuracy points for \textsc{Qwen2.5-0.5B} and a drop of 0.5 points for \textsc{Qwen2.5-1.5B}. These results indicate that \ourmethod preserves standard dense behavior rather than trading it off for compressibility.



\looseness=-1
\textbf{Compression under fixed KV-optimization budget (Q2).} We evaluate compression quality at different keep ratios, under a fixed KV-optimization budget. To do this, we sample held–out prefix--suffix pairs from FineWeb and directly optimize a compacted prefix KV cache to produce the corresponding suffix. We optimize the KV slots via Attention Matching using 256 query states and 8 NNLS iterations across all keep ratios. We report three different similarity metrics between the suffix next token distribution under the full/compressed prefix, averaged over 128 pairs per keep ratio. As seen in Table~\ref{tab:attention-matching-ppl}, under a fixed optimization budget, the \ourmethod-trained model consistently outperforms the base model across compression ratios, model sizes, and similarity measures, and achieves up to \textbf{3.21}$\times$ better retention of suffix perplexity. We additionally evaluate gradient-based KV cache optimization on the same prefix--suffix pairs. As seen in Figure~\ref{fig:qwen0p5_ppl_plots_main}, the \ourmethod-trained model achieves better suffix perplexity retention throughout optimization, requiring up to \textbf{5}$\times$ fewer optimization steps to match the base model's performance. For additional details, see Appendix~\ref{app:suffix} and for cross-domain evaluation see Appendix~\ref{app:cross-domain-am}.

\begin{wraptable}[14]{l}{0.70\linewidth}
  \vspace{-16pt}
  \centering
  \small
  \setlength{\tabcolsep}{4pt}
  \caption{\textbf{Long-form QA after context compression.} We evaluate the models on seven tasks from LongBench v2. The KV caches of the questions' contexts are compressed with a gradient-based method.}
  \label{tab:lbv2-step20-selected-monotone}
  \vspace{4pt}
  \resizebox{\linewidth}{!}{
      \begin{tabular}{lcccccc}
      \toprule
      & \multicolumn{2}{c}{10\% keep} & \multicolumn{2}{c}{20\% keep} & \multicolumn{2}{c}{50\% keep} \\
      \cmidrule(lr){2-3} \cmidrule(lr){4-5} \cmidrule(lr){6-7}
      Subdomain & Base & \ourmethod & Base & \ourmethod & Base & \ourmethod \\
      \midrule
      Academic & 18.1 & \textbf{26.6} & 19.1 & \textbf{27.7} & 21.3 & \textbf{26.6} \\
      Agent history QA & 20.0 & \textbf{30.0} & 25.0 & \textbf{30.0} & \textbf{30.0} & \textbf{30.0} \\
      Knowledge graph reasoning & \textbf{33.3} & \textbf{33.3} & \textbf{26.7} & \textbf{26.7} & 20.0 & \textbf{26.7} \\
      Legal & 27.3 & \textbf{48.5} & 27.3 & \textbf{45.5} & 27.3 & \textbf{42.4} \\
      Many-shot learning & \textbf{28.6} & 19.0 & \textbf{28.6} & \textbf{28.6} & \textbf{38.1} & 33.3 \\
      New language translation & \textbf{30.0} & \textbf{30.0} & 15.0 & \textbf{35.0} & 15.0 & \textbf{35.0} \\
      Table QA & 11.1 & \textbf{27.8} & \textbf{22.2} & \textbf{22.2} & \textbf{38.9} & 33.3 \\
      \midrule
      \textbf{Total} & 22.2 & \textbf{30.3} & 22.2 & \textbf{30.8} & 25.3 & \textbf{31.2} \\
      \bottomrule
      \end{tabular}
  }
\end{wraptable}
\looseness=-1
\textbf{Retrieval from a compressed context (Q3).}
We conduct a needle-in-a-haystack experiment in which we first optimize a compact KV cache for the haystack prefix, and then evaluate exact-match passkey retrieval from the resulting compressed cache. We compress the haystack using the gradient-based method described earlier, training the compacted KVs for 100 steps on the reconstruction sequence. We find that \ourmethod improves retrieval accuracy over the base model, particularly at moderate compression budgets. \ourmethod\ increases mean retrieval accuracy by 6.4 points for \textsc{Qwen2.5-0.5B} and 5.2 points for \textsc{Qwen2.5-1.5B} (Table~\ref{tab:niah}), achieving an 11--19 point improvement at 30--50 percent keep ratios. See additional details in Appendix~\ref{app:niah-kv-compression}.

\looseness=-1
\textbf{Compressed long-context question answering (Q4).}
We evaluate on seven LongBench v2~\citep{bai2025longbench} tasks, where the context of each question is compressed using the reconstruction sequence before model prediction. As shown in Table~\ref{tab:lbv2-step20-selected-monotone}, compared to the base model, the \ourmethod-trained model achieves better average accuracy across all KV retention ratios, with an average improvement of up to \textbf{39\%}. These results demonstrate that \ourmethod yields consistent downstream gains, even when applied to LLMs using only an NTP objective and no post-training. See additional details in Appendix~\ref{app:lbv2-cache-reconstruction}.

%% file: sections/related_works.tex
\section{Related Work}

\looseness=-1
\textbf{Token-level context compression and retrieval.}
One approach to long-context inference reduces the number of input tokens before prompting, rather than modifying the model's hidden state. Retrieval-augmented generation keeps the prompt short by selecting only relevant external text~\citep{lewis2020rag}, while prompt-compression methods prune or rewrite the provided context in token space~\citep{jiang2023llmlingua,jiang2024longllmlingua,li2023compressing, oai, antho}.
These methods are complementary to latent KV cache compression.

\textbf{Learned memory and soft-token compression.}
Another line of work uses learned memory or soft tokens as compact substitutes for long contexts. Recurrent Memory Transformers add memory tokens that carry information between segments~\citep{bulatov2022rmt}; AutoCompressors adapt language models to map earlier segments into summary vectors that act as soft prompts~\citep{chevalier2023adapting}; and gist-tokens train models to compress prompts into reusable learned tokens~\citep{mu2023gist}. DuoAttention instead learns which attention heads require full retrieval over the KV cache, and which can run with a streaming-style state~\citep{xiao2025duoattention}. These methods modify the model's runtime interface, architecture, or cache policy.  Another direction~\citep{chen2024generative,charakorn2026doc,liu2026shine} uses hypernetworks to encode long contexts into model parameters, typically via LoRA adapters, rather than storing them in the prompt or KV cache.

\textbf{Post-hoc KV cache compression.}
Another line of work focuses on reducing the KV cache of pretrained transformers at inference time. Token-selection and eviction methods keep a subset of original KV slots using attention, recency, or query-dependent importance signals~\citep{zhang2023h2o,oren2024tova,li2024snapkv,cai2025pyramidkv,tang2024quest}. KVzip selects keys using reconstruction-style objectives over the context~\citep{kim2025kvzip}. Other post-hoc methods move beyond choosing a subset of the original tokens. Cartridges optimize a compact latent KV cache for each context~\citep{eyuboglu2025cartridges}, Attention Matching constructs compact keys and values to preserve attention behavior~\citep{zweiger2026fastkv}, and Lexico represents KV vectors with sparse codes over learned universal dictionaries~\citep{kim2025lexico}. Our work is orthogonal to these compressors: we ask whether the model can be trained so that the same post-hoc methods work better.

%% file: sections/conclusion.tex
\section{Conclusion}\label{sec:conclusion}

In this work, we argue that KV cache compressibility is not only a property of the task or input, but also of the representation learned by the transformer. Theoretically, we show  that almost any sequence-to-vector function admits both highly compressible and inherently non-compressible implementations. We also provide a motivating example which further suggests that standard training may favor simpler but non-compressible solutions over more structured, compressible ones, thereby motivating compression-aware training. Guided by these insights, we introduce \ourmethod, a compression-aware training procedure that promotes compressible representations. By incorporating a train-time KV sparsification policy with a fixed retention budget into the forward pass, \ourmethod exposes the model to compression constraints during training, encouraging representations that remain effective under cache compression. Empirically, we demonstrate that models trained with \ourmethod consistently improve the quality–budget tradeoff of state-of-the-art post-hoc compression methods, at matched cache budgets over a range of tasks.  Overall, our results suggest a complementary approach to post-hoc compression methods: training models to become compressible.

%% file: sections/acks.tex
\section{Acknowledgements}
YG is supported by the UKRI Engineering and Physical Sciences Research Council (EPSRC) CDT in Autonomous and Intelligent Machines and Systems (grant reference EP/S024050/1). MB is partially supported by the EPSRC Turing AI World-Leading Research Fellowship No. EP/X040062/1 and EPSRC AI Hub No. EP/Y028872/1. HM is supported by the Israel Science Foundation through a personal grant (ISF 264/23) and an equipment grant (ISF 532/23).

%% file: appendix/theory.tex
\section{Theory}
\label{app:theory}
\subsection{Transformers and KV Cache Compression}
\label{app:preliminaries}

\textbf{Notation.} 
For a set $\alb$, let $\alb^*$ denote the set of all finite sequences over $\alb$. We denote by $\ve_i$ the $i$-th standard (one hot encoded) basis vector. For $n \in \sN$, we write $[n] = \{1,\dots,n\}$.

For matrices $\mX \in \sR^{n \times d}$ and $\mY \in \sR^{n \times d'}$, we denote their row-wise concatenation (along the feature dimension) by
\begin{equation}
[\mX,\mY] \in \sR^{n \times (d + d')}.
\end{equation}

For matrices $\mX \in \sR^{n \times d}$ and $\mZ \in \sR^{n' \times d}$, we denote their column-wise concatenation (along the sequence dimension) by
\begin{equation}
\begin{bmatrix}
\mX \\
\mZ
\end{bmatrix}
\in \sR^{(n + n') \times d}.
\end{equation}

When convenient, we sometimes write  $[\mX^\top,\mZ^\top]^\top$ instead of $
\begin{bmatrix}
\mX \\
\mZ
\end{bmatrix}.$

For vectors $\vv, \vu \in \sR^n$, we treat them as column vectors in $\sR^{n \times 1}$ and write
\begin{equation}
[\vv,\vu] \in \sR^{n \times 2}.
\end{equation}

\begin{definition}[Attention Head]
Let $d_{\mathrm{in}}, d_{\mathrm{out}} \in \mathbb{N}$. 
An \emph{attention head} is a function
\begin{equation}
\mathrm{AttnHead} : (\sR^{d_{\mathrm{in}}})^* \to (\sR^{d_{\mathrm{out}}})^*
\end{equation}
defined as follows. For a sequence $\mX = [\vx_1, \dots, \vx_n]^\top \in \mathbb{R}^{n \times d_{\mathrm{in}}}$,
\begin{equation}
\label{eq:kv_compute}
\mQ = \mX \mW_Q,\quad \mK = \mX \mW_K,\quad \mV = \mX \mW_\mV,
\end{equation}
where $\mW_Q, \mW_K, \mW_V \in \mathbb{R}^{d_{\mathrm{in}} \times d_{\mathrm{out}}}$.
The output is
\begin{equation}
\label{eq:attention}
\mathrm{AttnHead}(X)
=
\mathrm{softmax}\!\left(\frac{\mQ \mK^\top}{\sqrt{d_k}}\right)\mV,
\end{equation}
where the softmax is applied row-wise.
\end{definition}

\begin{definition}[Transformer Block]
Fix $h, d_{\mathrm{in}}, d_{\mathrm{out}}, d_{\mathrm{ff}} \in \sN$. 
A \emph{transformer block} is a function
\begin{equation}
\block : (\sR^{d_{\mathrm{in}}})^* \to (\sR^{d_{\mathrm{out}}})^*
\end{equation}
defined as follows. For a sequence
\begin{equation}
\mX = [\vx_1, \dots, \vx_n]^\top \in \sR^{n \times d_{\mathrm{in}}},
\end{equation}
we define
\begin{equation}
\block(\mX)
=
\ffn\!\left(
[\mathrm{head}_1(\mX), \dots, \mathrm{head}_h(\mX)] \mW_O
+ \mX\right),
\end{equation}
where each $\mathrm{head}_i$ is an attention head with input and output dimensions $d_{\mathrm{in}}, d_{\mathrm{out}}$ respectively, and 
$\mW_O \in \sR^{h \cdot d_{\mathrm{out}} \times d_{\mathrm{out}}}$, and $\ffn$ is a 2-layer feed forward network applied row-wise, i.e.,
\begin{equation}
\ffn(\mX)_i = \ffn(\vx_i), \quad i = 1,\dots,n,
\end{equation}
with
\begin{equation}
\ffn(\vx) = \sigma(\vx \mW_1 + \vb_1)\mW_2 + \vb_2,
\end{equation}
where $\mW_1 \in \sR^{d_{\mathrm{out}} \times d_{\mathrm{ff}}}$,
$\mW_2 \in \sR^{d_{\mathrm{ff}} \times d_{\mathrm{out}}}$, 
$\vb_1 \in \sR^{d_{\mathrm{ff}}}$, and $\vb_2 \in \sR^{d_{\mathrm{out}}}$.
\end{definition}

\textbf{Note.} In the theoretical analysis, we omit layer normalization, causal masking, and the residual connection following the FFN for simplicity, as is standard in prior theoretical treatments (See e.g. ~\cite{sanford2023representational, yehudai2024can}). Our results extend to the full architecture with only minor modifications. All empirical evaluations are conducted with these components included.

\begin{definition}[Token Embedding]
Let $\alb$ be a finite alphabet. A \emph{token embedding} function is a function
\begin{equation}
\te : \alb \times \sN \to \sR^{d_{\mathrm{model}}}.
\end{equation}
Given a sequence $\va = (a_1,\dots,a_n) \in A^*$, its embedding is the sequence
\begin{equation}
\te(\va) = [\te(a_1,1), \dots, \te(a_n,n)]^\top \in \sR^{n \times d_{\mathrm{model}}}.
\end{equation}
\end{definition}

\begin{definition}[Transformer]
Let $\alb$ be a finite alphabet. A \emph{transformer} is a tuple
\begin{equation}
\model = (\block_1, \dots, \block_L, \te),
\end{equation}

where each $\block_\ell : (\sR^{d_{\ell-1}})^* \to (\sR^{d_{\ell}})^*$ is a transformer block, and $\te : A^* \to (\sR^{d_0})^*$ is a token embedding function.

The transformer $\model$ defines a function
\begin{equation}
\model : A^* \to \sR^{d_L}
\end{equation}
as follows. For $\va = (a_1, \dots, a_n) \in A^*$, let
\begin{equation}
\mX^{(0)} = \te(\va), \qquad 
\mX^{(\ell)} = \block_\ell(\mX^{(\ell-1)}), \quad \ell = 1,\dots,L.
\end{equation}
Then
\begin{equation}
\model(\va) = \mX^{(L)}_n,
\end{equation}
i.e., the output is the representation of the final token.
\end{definition}

\begin{definition}[KV Cache Compression]
\label{def:kv_compression}

A \emph{KV cache compression policy} is a tuple $\compress = (\cl_1, \dots, \cl_L)$, where each
\begin{equation}
\cl_\ell : (\sR^{d_\ell} \times \sR^{d_\ell})^* \to (\sR^{d_\ell} \times \sR^{d_\ell})^*
\end{equation}
maps a sequence of key--value pairs to a shorter sequence. 

Concretely, for
\begin{equation}
(\mK, \mV) = \bigl([\vk_1,\dots,\vk_n]^\top, [\vv_1,\dots,\vv_n]^\top\bigr),
\end{equation}
we write
\begin{equation}
\cl_\ell(\mK,\mV) = (\tilde{\mK}, \tilde{\mV})
= \bigl([\tilde{\vk}_1,\dots,\tilde{\vk}_{\length(n)}]^\top, [\tilde{\vv}_1,\dots,\tilde{\vv}_{\length(n)}]^\top\bigr),
\end{equation}
where $\length(n) \le n$. The function $\length(\cdot)$ is called the \emph{compression budget}.

\vspace{0.5em}
Given a transformer $\model = (\te, \block_1,\dots,\block_L)$, a compression policy $\compress$, and a \emph{context} sequence $\va = (a_1,\dots,a_n)$, we define the associated \emph{compressed transformer} $\model_{\compress,\va}$.

For an input sequence $\vb = (b_1,\dots,b_k)$, the output $\model_{\compress,\va}(\vb)$ is obtained by running the forward pass of $\model$ on the concatenated sequence $[\va,\vb] = (a_1,\dots, a_n, b_1,\dots,b_k)$, with the following modifications.



For each $\ell$ and each attention head at the $\ell$-th block, after computing keys and values, the compression function $\cl_\ell$ is applied to the KV pairs corresponding to the context tokens. The compressed KV pairs replace the original ones in the attention computation.

\vspace{0.5em}
More precisely, let $(\mK_\va, \mV_\va)$ denote the KV cache produced by the prefix $\va$ at some attention head of $\block_\ell$ of the original model $\model$, and let $\mY \in \sR^{k \times d_{\ell-1}}$ be the input corresponding to $\vb$ in the forward pass computation of $\model_{\compress, \va}$ after block $\ell-1$. We first compute
\begin{equation}
\label{eq:kv_compute2}
\mQ_\vb = \mY \mW_Q, \quad
\mK_\vb = \mY \mW_K, \quad
\mV_\vb = \mY \mW_V.
\end{equation}
We then compress the context KV pairs:
\begin{equation}
(\tilde{\mK}_\va, \tilde{\mV}_\va) = \cl_\ell(\mK_\va, \mV_\va).
\end{equation}
The resulting attention computation is
\begin{equation}
\label{eq:attention2}
\mathrm{AttnHead}(\mY)
=
\mathrm{softmax}\!\left(
\frac{1}{\sqrt{d_k}}
\mQ_\vb
\begin{bmatrix}
\tilde{\mK}_\va \\
\mK_\vb
\end{bmatrix}^\top
\right)
\begin{bmatrix}
\tilde{\mV}_\va \\
\mV_\vb
\end{bmatrix}.
\end{equation}

\end{definition}

\begin{definition}\label{def:compressibility2}[KV compressibility]
Let $N \in \sN$, $\varepsilon > 0$, and let $\length : \sN \to \sN$ be a budget function. A transformer $\model$ is said to be $(N, \varepsilon, \length)$-compressible if there exists a KV cache compression policy $\compress$ with budget $\length$ such that for every pair of sequences $\va, \vb$ of lengths $n$ and $k$ with combined length satisfying $n + k \le N$, it holds that
\begin{equation}
\| \model([\va, \vb]) - \model_{\compress, \va}(\vb) \| < \varepsilon.
\end{equation}
\end{definition}

We include $N$ explicitly in the definition above, as many natural sequence-to-vector functions require model dimension scaling with $N$ (e.g., $O(N)$ or $O(\log N)$) in order to achieve arbitrarily accurate approximation; see \citet{sanford2023representational, yehudai2024can} for examples.

\subsection{Motivating example: histogram computation}
\label{app:motivating_example}
In this section we provide formal proofs of Propositions \ref{prop:hist_non_compressible} and \ref{prop:hist_compressible}.

\begin{proof}[Proof of Proposition \ref{prop:hist_non_compressible}]
The construction of $\model = (\te, \block_1, \block_2)$ is straightforward. First, for $j \in [m]$ and $i \in [N]$, define
\begin{equation}
\te(j,i) = \vu_{j} + \vp_i,
\end{equation}
where $\vu_j, \vp_i \in \sR^m$ encode token value and token position respectively, and assume that $\te$ is injective over $[m] \times [N]$.

Let $\mX \in \sR^{n \times m}$. We define $\block_1 (\sR^{m})^* \to (\sR^{2m})^*$ by
\begin{equation}
\block_1(\mX)_i \approx \rho_1(\mX_i)
\end{equation}
where $\rho_1:\sR^{m} \to \sR^{2m}$ is applied row wise $\mX$ and satisfies for each $j,i \in [m] \times [N]$:

\begin{equation}
    \rho_1(\vu_j +\vp_i) = \begin{bmatrix}\ve_{j}\\ \mathbf{0}\end{bmatrix}. 
\end{equation}

$\block_1$ can be implemented by zeroing out all attention projections and relying on the residual connection together with a feedforward network that approximates $\rho_1$ to arbitrary precision (this can be obtained as a two-layer feedforward network is a universal approximator of continuous functions on compact sets \citep{cybenko1989approximation,hornik1991approximation}).


Let $\mX \in \sR^{n \times 2m}$, and write each row as $\vx_i$. We define
\begin{equation}
\block_2(\mX)_i = \frac{1}{n} \sum_{j=1}^n \vx_j
\end{equation}
i.e., each output token is replaced by the average of all input tokens. This operation can be implemented using a single attention head by setting
\begin{equation}
\mW_Q = \mW_K = \mathbf{0},
\end{equation}
so that all attention weights are uniform, and choosing
\begin{equation}
\mW_V = \mW_O =
\begin{pmatrix}
\mathbf{0} & \mathbf{0} \\
\mI & \mathbf{0}
\end{pmatrix}.
\end{equation}
This choice extracts the first half of each vector and places it in the second half before averaging. After applying a residual connection, the intermediate representation of the $i$-th token is then

\begin{equation}
\begin{bmatrix} \vx_i \\ \frac{1}{n} \sum_{j=1}^n \vx_j
    
\end{bmatrix}
\end{equation}

Finally, the feedforward network is taken to be 
\begin{equation}
\ffn(\vx, \vy) = \vy,
\end{equation}

resulting in a final representation of $\frac{1}{n}\sum_{i=1}^n \vx_i$. For input sequence $\va$ of length $n$ we thus have 
\begin{equation}
    \model(\va) \approx \frac{1}{n}\sum_{i=1}^n \ve_{a_i}  =\left(
\frac{1}{n}\sum_{i=1}^n \mathbf{1}_{a_i=1},\;
\frac{1}{n}\sum_{i=1}^n \mathbf{1}_{a_i=2},\;
\dots,\;
\frac{1}{n}\sum_{i=1}^n \mathbf{1}_{a_i=m}
\right) = \hist(\va),
\end{equation}
Thus $\model$ approximates $\hist$ to arbitrary precision. Now let $\compress = (\cl_1, \cl_2)$ be any compression policy with budget function $\length(\cdot)$ satisfying $\length(N-1) < N-1$. Fix a prefix sequence $\va$ of length $n$ and a suffix sequence $\vb$ of length $k$, and consider the computation of $\model_{\compress,\va}(\vb)$.

First, observe that in the first block all attention projection matrices (and in particular $\mW_O$) are zero. Hence, $\cl_1$ has no effect on the representations of the tokens of $\vb$, and after the first block the representations are exactly as in the uncompressed model.

Next, consider the second block. Let $(\mK_\va,\mV_\va)$ denote the KV cache corresponding to the prefix $\va$ computed durning the computation of $\model(\va)$, and let $(\tilde{\mK}_\va,\tilde{\mV}_\va) = \cl_2(\mK_\va,\mV_\va)$ be the compressed cache. Since $\mW_Q = 0$, all attention scores are uniform for both $\model$ and $\model_{\compress, \va}$ and so

\begin{equation}
    \model_{\compress, \va}(\vb) = \frac{1}{\length(n)+k} \left( \sum_{i=1}^{\length(n)} \tilde \vv_i + \sum_{i=1}^k \ve_{b_i} \right).
\end{equation}

\begin{equation}
    \model([\va,\vb]) = \frac{1}{n+k} \left( \sum_{i=1}^{n} \ve_{a_i} + \sum_{i=1}^k \ve_{b_i} \right).
\end{equation}

For notational convenience, define
\begin{equation}
\tilde \vv = \frac{1}{\length(n)+k} \sum_{i=1}^{\length(n)} \tilde \vv_i, \quad
\tilde \va = \frac{1}{n+k} \sum_{i=1}^{n} \ve_{a_i}, \quad
\tilde \vb =  \sum_{i=1}^{k} \ve_{b_i}.
\end{equation}

Then the difference between the compressed and full outputs can be written as

\begin{equation}
    \| \model_{\compress, \va}(\vb) - \model([\va,\vb]) \|
    =
    \left\|
    \tilde \vv - \tilde \va
    + \left(\frac{1}{\length(n)+k} - \frac{1}{n+k}\right) \tilde \vb
    \right\|.
\end{equation}

We now derive a lower bound. First, choose $\vb = (1,1,\dots,1)$ and examine the first coordinate. This yields
\begin{equation}
    \| \model_{\compress, \va}(\vb) - \model([\va,\vb]) \|
    \geq
    \left|
    \tilde \vv_1 - \tilde \va_1
    + \frac{k(\length(n)-n)}{(n+k)(\length(n)+k)}
    \right|.
\end{equation}
Next, choosing $\vb = (2,2,\dots,2)$ removes the contribution of the last term, giving
\begin{equation}
    \| \model_{\compress, \va}(\vb) - \model([\va,\vb]) \|
    \geq
    |\tilde \vv_1 - \tilde \va_1|.
\end{equation}

Combining the two inequalities, we conclude that for any choice of $\tilde \vv$, there exists a $\vb$ such that
\begin{equation}
    \| \model_{\compress, \va}(\vb) - \model([\va,\vb]) \|
    \geq
    \frac{k(n - \length(n))}{2(n+k)(\length(n)+k)}.
\end{equation}
Finally, taking $n = N-1$ and $k=1$ yields
\begin{equation}
    \| \model_{\compress, \va}(\vb) - \model([\va,\vb]) \| \geq \frac{N-1 - \length(N-1)}{2(N)(\length(N-1)+1)} = \const > 0,
\end{equation}
which completes the proof.
\end{proof}

\textbf{Note:} for the weakest compression case $\length(n) = n-1$, the constant above is $\const = O(N^{-2})$, and so as $N$ grows the approximation bound becomes weaker, but for the extreme compression case $\length(n) = c$ for some $c \in \sN$, $\const$ can be chosen independently of $N$.

\begin{proof}[proof of Proposition \ref{prop:hist_compressible}]

We construct $\model = (\te, \block_1, \block_2)$ similarly to the above construction, with a few modifications. Like before, for $j \in [m]$ and $i \in [N]$, define
\begin{equation}
\te(j,i) = \vu_{j} + \vp_i,
\end{equation}
where $\vu_j, \vp_i \in \sR^m$ encode token value and token position respectively, and assume that $\te$ is injective over $[m] \times [N]$.

Let $\mX \in \sR^{n \times m}$. We define $\block_1 (\sR^{m})^* \to (\sR^{4m})^*$ by
\begin{equation}
\block_1(\mX)_i \approx \rho_1(\mX_i)
\end{equation}
where $\rho_1:\sR^{m} \to \sR^{4m}$ is applied row wise $\mX$ and satisfies for each $j,i \in [m] \times [N]$:

\begin{equation}
    \rho_1(\vu_j +\vp_i) = \begin{bmatrix}\ve_{j}\\ 
    \vp_i\\
    \mathbf{0}\\
    \mathbf{0}\end{bmatrix}. 
\end{equation}

Where each of the vector blocks above is of size $m$. That is, $\block_1$ is identical to the construction in the proof of Proposition~\ref{prop:hist_non_compressible}, except that it preserves the positional component. As before, $\block_1$ can be implemented by zeroing out all attention projections and relying on the residual connection together with a feedforward network that approximates $\rho_1$ to arbitrary precision. This follows from the universal approximation property of two-layer feedforward networks for continuous functions on compact sets \citep{cybenko1989approximation,hornik1991approximation}.

We construct $\block_2$ to be composed of a single attention head with

\begin{equation}
\mW_Q = \mW_K = \mathbf{0}, \mW_O = \mI
\end{equation}
so that all attention weights are uniform, and choosing
\begin{equation}
\mW_V =
\begin{pmatrix}
\mI & \mathbf{0} & \mathbf{0} & \mathbf{0}\\
\mathbf{0} & \mI & \mathbf{0} & \mathbf{0}\\
\mI & \mathbf{0} & \mathbf{0} & \mathbf{0}\\
\mathbf{0} & \mathbf{0} & \mathbf{0} & \mathbf{0}
\end{pmatrix}.
\end{equation}

This choice extracts the first m-length block of each vector and places it in the third block before averaging, keeping the first and second vector blocks the same, and zeroing out the last vector block. After applying a residual connection, the intermediate representation of the $i$-th token $a_i$ is then

\begin{equation}
\label{eq:attention_compressed_hist}
\begin{bmatrix} 
\ve_{a_i} \\ 
\vp_i \\ 
\frac{1}{n} \sum_{j=1}^n \ve_j\\
\mathbf{0}
\end{bmatrix}.
\end{equation}

finally, we choose the last feed forward network to approximate a map $\rho_2$  satisfying for all $i,j,k \in [N]$:

\begin{equation}
\label{eq:ff_compressed_hist}
    \rho_2(\vw, \vx,\vy,\vz) = \begin{cases}
        \vy & \text{ if } \vz=0 \\
        \frac{j - i + 1}{j}\vy  & \text{ if } \vz=\frac{\vp_i}{k+1} \text{ and } \vx = \vp_j.
    \end{cases}
\end{equation}

Such a function can be realized to arbitrary precision by a feedforward network: the above specification is defined on a union of disjoint compact sets, and therefore admits a continuous extension to a compact domain, which can in turn be approximated by a standard feedforward network via universal approximation. 

Equations~\ref{eq:attention_compressed_hist} and~\ref{eq:ff_compressed_hist} imply that, in the uncompressed setting, for any sequence $\va$ of length $n \le N$, we recover as before

\begin{equation}
    \model(\va) \approx \frac{1}{n}\sum_{i=1}^n \ve_{a_i}  =\left(
\frac{1}{n}\sum_{i=1}^n \mathbf{1}_{a_i=1},\;
\frac{1}{n}\sum_{i=1}^n \mathbf{1}_{a_i=2},\;
\dots,\;
\frac{1}{n}\sum_{i=1}^n \mathbf{1}_{a_i=m}
\right) = \hist(\va),
\end{equation}

and so $\model$ approximates $\hist$.  We now define a compression policy $\compress = (\cl_1, \cl_2)$. We first note that, like before, SINCE $\mW_O=\mathbf{0}$ in the first block,  $\cl_1$ does not effect the forward pass of $\model_{\compress, \va}$.

Given a prefix $\va$ of length $n$ with KV cache $(\mK_\va,\mV_\va)$, we compress it into a single KV pair:
\begin{equation}
\cl_2(\mK_\va,\mV_\va) = \Bigl( \mathbf{0},
\begin{bmatrix}
\sum_{i=1}^n \ve_{a_i} \\
\mathbf{0} \\
\mathbf{0} \\
\vp_n
\end{bmatrix}
\Bigr).
\end{equation}
That is, we store a single value vector containing (i) the unnormalized histogram of the prefix and (ii) its length encoded via $\vp_n$. Now consider processing a suffix $\vb = (b_1,\dots,b_k)$. Since the compressed prefix consists of a single KV entry with zero key, attention again produces a uniform average over the compressed prefix and suffix tokens. The representation of the final token $b_k$ becomes
\begin{equation}
\label{eq:main_compressable_hist_rep_clean2}
\begin{bmatrix}
\ve_{b_i} \\
\vp_{n+k} \\
\frac{1}{k+1} \Bigl(\sum_{i=1}^n \ve_{a_i} + \sum_{i=1}^k \ve_{b_i}\Bigr) \\
\frac{1}{k+1}\vp_n
\end{bmatrix}.
\end{equation}

Applying $\rho_2$ and using \eqref{eq:ff_compressed_hist}, we obtain
\begin{equation}
\model_{\compress, \va}(\vb) \approx \frac{1 + (n+k) - n}{n+k}  \cdot \frac{1}{k+1}\Bigl(\sum_{i=1}^n \ve_{a_i} + \sum_{i=1}^k \ve_{b_i}\Bigr)
=
\frac{1}{n+k} \sum_{i=1}^{n+k} \ve_{c_i},
\end{equation}
where $c_i$ ranges over $[\va,\vb]$. This matches the histogram of the concatenated sequence (up to $\varepsilon$), completing the proof.
    
\end{proof}

\subsection{Compressible Transformers for General Functions}
\label{app:main_theorem}

In this section, we present a formal proof of Theorem~\ref{thm:main_theorem}. The proof is divided into two lemmas: the first (Lemma \ref{lemma:compressible_main}) establishes the existence of compressible transformer approximations, and the second (Lemma \ref{lemma:non_compressible_main} demonstrates the existence of non-compressible ones. The theorem then follows immediately.

\begin{lemma}
\label{lemma:compressible_main}
    Let $\alb$ be a finite alphabet, and let $\func : \bigcup_{n \le N} \alb^n \to \mathbb{R}^{d_{\text{out}}}$ be any sequence-to-vector function. Then for every $\varepsilon > 0$ There exists a transformer such that

    \begin{enumerate}
    \item (\textbf{Approximation}) For every sequence $\va = (a_1,\dots,a_n) \in A^n$ with $n \le N$,
    \begin{equation}
    \|\func(\va) - \model(\va)\| < \varepsilon.
    \end{equation}
    
    \item (\textbf{Maximal compressibility}) There exists a KV cache compression policy $\compress$ with compression budget
    \begin{equation}
    \length(n) \equiv 1
    \end{equation}
    such that for every prefix $\va \in A^n$ and suffix $\vb \in A^k$ with $k+n \le N$,
    \begin{equation}
    \label{eq:fully_compressible}
    \| \model([\va, \vb]) - \model_{\compress,\va}(\vb) \| < \varepsilon.
    \end{equation}
\end{enumerate}
    
\end{lemma}



    


\begin{proof}

We assume that the token embedding map $\te : \alb \times \sN \to \sR^{d_0}$ is defined by

\begin{equation}
\te(a,i) = \vu_a + \vp_i
\end{equation}

where $\|\vp_i\| = \|\vp_j\|$ for all $i,j \in [N]$, and $\vu_a, \vp_i \in \sR^{d_0}$ for all $a \in A$. We further assume injectivity of the embedding function, that is

\begin{equation}
(a,i) \neq (b,j) \;\Rightarrow\; \te(a,i) \neq \te(b,j).
\end{equation}

Such embeddings can be obtained, for example, using a learned lookup table for tokens combined with a positional encoding (e.g., sinusoidal) with sufficiently large frequencies.

Let $\vu_\emptyset \in \sR^{d_0}$ be a padding vector such that $\vv_\emptyset \neq \te(a,i)$ for all $a \in A$ and $i \in \sN$. For a sequence $\va = (a_1,\dots,a_n) \in A^*$ with $n \le N$, define
\begin{equation}
U_{\va}
=
\left\{
\mP [\te(a_1,1), \dots, \te(a_n,n), \vu_\emptyset, \dots, \vu_\emptyset]^\top \mid \mP \in \sR^{N \times N} \text{ is a permutation matrix}
\right\}
\subset \sR^{N \times d_0}.
\end{equation}
That is, we embed the sequence, pad it to length $N$ using $\vu_\emptyset$, and then apply an arbitrary permutation to the order of the elements.

By injectivity of the embeddings, the sets $\{U_{\va} : \va \in A^*,\, |\va|\le N\}$ are pairwise disjoint. Consequently, there exists a continuous function
\begin{equation}
\bar{\func} : \sR^{N \times d_0} \to \sR^{d_{\mathrm{out}}}
\end{equation}
such that for every $\va \in A^*$ with $|\va|\le N$ and every $[\vv_1, \dots, \vv_N]^\top \in U_{\va}$,
\begin{equation}
\bar{\func}(\vv_1, \dots, \vv_N) = \func(\va).
\end{equation}

Without loss of generality, we may assume that $\bar{\func}$ is permutation-invariant. That is, for every permutation $\sigma \in S_N$ and every $\mV = [\vv_1,\dots,\vv_N]^\top \in \sR^{N \times d_0}$,
\begin{equation}
\bar{\func}(\vv_1,\dots,\vv_N)
=
\bar{\func}(\vv_{\sigma(1)},\dots,\vv_{\sigma(N)}).
\end{equation}

Indeed, if $\bar{\func}$ is not permutation-invariant, we can define
\begin{equation}
\tilde{\func}(\vv_1,\dots,\vv_N)
=
\frac{1}{|S_N|}
\sum_{\sigma \in S_N}
\bar{\func}(\vv_{\sigma(1)},\dots,\vv_{\sigma(N)}).
\end{equation}
Then $\tilde{\func}$ is continuous and permutation-invariant. Moreover, for every sequence $\va \in A^*$ and every $\mV \in U_{\va}$,
\begin{equation}
\tilde{\func}(\mV) = \bar{\func}(\mV),
\end{equation}
since $U_{\va}$ is closed under permutations.

Because $\bar{\func}$ is continuous and permutation-invariant, it follows from \citet{zaheer2017deep} that there exists a pair of  functions 
$\phi : \sR^{d_0} \to \sR^{d_1}$ and $\rho : \sR^{d_1} \to \sR^{d_{\mathrm{out}}}$ such that
\begin{equation}
\bar{\func}(\vv_1,\dots,\vv_N)
=
\rho\!\left(\sum_{i=1}^N \phi(\vv_i)\right)
\end{equation}

(Note that some functions require $d_1 = O(N)$). We further assume without loss of generality that $\phi(\vv_\emptyset) = 0$. We now construct a compressible transformer $\model = (\te, \block_1, \block_2)$ approximating $\func$ using this decomposition. The construciton of this transformer cloesly resembles the one used in the proof of Proposition \ref{prop:hist_compressible}.

For the first block $\block_1$, we set all attention projection matrices to zero,
\begin{equation}
\mW_Q = \mW_K = \mW_V = \mW_O = \mathbf{0},
\end{equation}
so that the block reduces to its feedforward component.

Let $\tilde{\phi} : \sR^{d_0} \to \sR^{4d_0}$ be a continuous function satisfying
\begin{equation}
\tilde{\phi}(\te(a,i)) =
\begin{bmatrix}
\phi(\te(a,i)) \\
\vp_i\\
\mathbf{0} \\
\mathbf{0} 
\end{bmatrix}.
\end{equation}
for all $a \in A$ and $i \in [N]$ where of the four vector blocks above are of size $d_0$. That is, $\tilde{\phi}$ computes $\phi$ on the embedding while preserving positional information and padding with additional zeroes. We choose the feedforward map $\ffn_1$ in this block to approximate $\tilde{\phi}$ to accuracy $\varepsilon' > 0$, to be specified later (This is possible as 2-layer feed forwand networks are universal approximators of continuous functions on compact sets \citep{cybenko1989approximation,hornik1991approximation}). Thus, for any sequence $\va = (a_1,\dots,a_n)$,
\begin{equation}
\label{eq:block1_forward}
\block_1(\te(\va))
\approx
\begin{bmatrix}
\tilde{\phi}(\te(a_1,1)) \\
\vdots \\
\tilde{\phi}(\te(a_n,n))
\end{bmatrix}.
\end{equation}

For the second block $\block_2$, we again use a single attention head, and set
\begin{equation}
\mW_Q = \mW_K = \mathbf{0}, \mW_O = \mI
\end{equation}

The value projection $\mW_V$ is chosen to extract the $\phi$-component from the output of $\tilde{\phi}$, and copy it to the third $d_1$-sized vector block,  keep the first $2 d_1$ coordinates the same, and zero out the remaining coordinates. That is:

\begin{equation}
\mW_V =
\begin{pmatrix}
\mI & \mathbf{0} & \mathbf{0} & \mathbf{0}\\
\mathbf{0} & \mI & \mathbf{0} & \mathbf{0}\\
\mI & \mathbf{0} & \mathbf{0} & \mathbf{0}\\
\mathbf{0} & \mathbf{0} & \mathbf{0} & \mathbf{0}
\end{pmatrix}.
\end{equation}




Since $\mW_Q = \mW_K = \mathbf{0}$, all attention scores are identical, and hence the attention weights are uniform.  Applying the attention update in $\block_2$ to the outputs of $\block_1$ and using a residual connection, the representation at position $i$ becomes approximately
\begin{equation}
\begin{bmatrix}
    \tilde{\phi}(\te(a_i,i))\\
    \vp_i\\
    \frac{1}{n} \sum_{j=1}^n \phi(\te(a_j,j)) \\
    \mathbf{0}
\end{bmatrix}.
\end{equation}

We now choose the feedforward network of $\block_2$ to approximate up to $\varepsilon''$ (to be chosen later) a continuous function $\tilde \rho$ satisfying

\begin{equation}
\label{eq:ff_compressed_geneal}
    \tilde \rho(\vw, \vx,\vy,\vz) = \begin{cases}
        \rho (j \cdot \vy) & \text{ if } \vz=0  \text{ and } \vx = \vp_j.\\
        \rho((j - i + 1) \cdot \vy)  & \text{ if } \vz=\frac{\vp_i}{k+1} \text{ and } \vx = \vp_j.
    \end{cases}
\end{equation}


Intuitively, this map recovers the position index $i$ from the positional embedding $\vp_i$, rescales the averaged sum as necessary to recover the unweighted sum, and applies $\rho$. In particular, at the final position $n$, the output of $\model$ is approximately
\begin{equation}
\rho\!\left(
n \cdot \frac{1}{n} \sum_{j=1}^n \phi(\te(a_j,j))
\right)
=
\rho\!\left(\sum_{j=1}^n \phi(\te(a_j,j))\right)
=
\func(\va).
\end{equation}

By choosing $\varepsilon', \varepsilon''$ sufficiently small, it follows that $\model$ approximates $\func$ to arbitrary precision.

We now show that there exists a compression policy with $\length(n) = 1$ that satisfies Equation~\ref{eq:fully_compressible}. 

Define $\compress = (\cl_1, \cl_2)$ as follows. For any sequence of key--value pairs $(\mK,\mV)$ with $\mV = [\vv_1,\dots,\vv_n]^\top$,

\begin{equation}
\cl_1(\mK,\mV) = (\mathbf{0}, \mathbf{0}),
\end{equation}
and
\begin{equation}
\cl_2(\mK,\mV) = (\tilde{\mK}, \tilde{\mV}),
\end{equation}
where $\tilde{\mK} = \mathbf{0}$ and
\begin{equation}
\tilde{\mV}
=
\begin{pmatrix}
\mI & \mathbf{0} & \mathbf{0} & \mathbf{0} \\
\mathbf{0} & \mathbf{0}& \mathbf{0} & \mathbf{0} \\
\mathbf{0} & \mathbf{0} & \mathbf{0} & \mathbf{0} \\
\mathbf{0} & \mathbf{0} & \mathbf{0} & \mathbf{0}
\end{pmatrix}
\sum_{i=1}^n \vv_i
\;+\;
\begin{pmatrix}
\mathbf{0} & \mathbf{0} & \mathbf{0} & \mathbf{0} \\
\mathbf{0} & \mathbf{0} & \mathbf{0} & \mI \\
\mathbf{0} & \mathbf{0} & \mathbf{0} & \mathbf{0} \\
\mathbf{0} & \mathbf{0} & \mathbf{0} & \mathbf{0}
\end{pmatrix}
\vv_n.
\end{equation}

clearly $\length(n)=1$.

For a prefix $\va$ of length $n$ and a suffix $\vb$ of length $k$, we begin by analyzing the forward pass of $\model_{\compress, \va}(\vb)$. After the first block, the representations of the tokens $b_i$ are given by 

\begin{equation}    
\begin{bmatrix}
\tilde{\phi}(\te(b_1,1+n)) \\
\vdots \\
\tilde{\phi}(\te(b_k,k+n))
\end{bmatrix}.
\end{equation}
This follows because the embedding layer of $\model_{\compress, \va}$ is identical to that of $\model$, applied to the concatenated sequence $[\va, \vb]$. Moreover, the computation of $\block_1$ is applied independently to each token representation, and is therefore unaffected by compression.
Next, consider the attention computation in the second block. Since $\mW_Q = 0$, all attention scores are zero, and hence the attention weights are uniform, equal to $\frac{1}{1+k}$. By the definition of $\cl_2$, it follows that after the attention update in $\block_2$, the representation of each token $b_i$ is

\begin{equation}
\begin{bmatrix}
    \phi(\te(b_i,n+i))\\
    \ve_{n+i}\\
    \frac{1}{1+k} \biggl( \sum_{j=1}^{n} \phi(\te(a_j,j)) + \sum_{j=1}^k \phi(\te(b_j,n+j)) \biggr)\\
    \frac{1}{k+1} \vp_n 
\end{bmatrix}.
\end{equation}
Finally, after applying the feed forward network $\ffn$ ,  by choosing $\varepsilon', \varepsilon''$ small enough, Equation \ref{eq:ff_compressed_geneal} implies that the final output of $\model_{\compress, \va}(\vb)$ is approximately 
\begin{equation}
 \rho\!\left(
(k+n-n+1) \cdot \frac{1}{k+1}  \bigl( \sum_{j=1}^n \phi(\te(a_j,j))
+ \sum_{j=1}^k \phi(\te(b_j,n+j)) \bigr) \right)
\approx
\func([\va, \vb])
\end{equation}

and is in an $\varepsilon$ approximator of  $\model_{\compress, \va}(\vb)$, completing the proof.
Finally, after applying the feedforward network $\ffn$, and choosing $\varepsilon', \varepsilon''$ sufficiently small, Equation~\ref{eq:ff_compressed_geneal} implies that the final output of $\model_{\compress, \va}(\vb)$ is approximately

and thus constitutes an $\varepsilon$-approximation of $\model_{\compress, \va}(\vb)$, completing the proof.

\end{proof}

\begin{lemma}
\label{lemma:non_compressible_main}
    Let $\alb$ be a finite alphabet, and let $\func : \bigcup_{n \le N} \alb^n \to \mathbb{R}^{d_{\text{out}}}$ be any sequence-to-vector function. Assume additionally that for  some prefix sequence $\bar \va$ of length $n<N$ there exists two suffix sequences $\vb^1$, $\vb^2$ both of length $k \le N-n$ such that 
    \begin{equation}
    \label{eq:compressabillity_condition}
        \func([\va, \vb^1]) \neq \func([\bar \va, \vb^2])
    \end{equation}

    Then for every $\varepsilon, \const > 0$ There exists a transformer such that

    \begin{enumerate}
    \item (\textbf{Approximation}) For every sequence $\va = (a_1,\dots,a_n) \in A^n$ with $n \le N$,
    \begin{equation}
    \label{eq:approximate_non_compressible}
    \|\func(\va) - \model(\va)\| < \varepsilon.
    \end{equation}
    
    \item (\textbf{Non- compressibility}) For every KV cache compression policy $\compress$ with compression budget satisfying $\length(n)<n$
    
    there exists a suffix $\vb \in A^k$ with $k+n \le N$ such that,
    \begin{equation}
    \label{eq:fully_compressible2}
    \| \model([\va, \vb]) - \model_{\compress,\va}(\vb) \| > \const.
    \end{equation}
\end{enumerate}
    
\end{lemma}



    

\begin{proof}

By the same argument as in the proof of Lemma~\ref{lemma:compressible_main}, there exist functions 
$\phi : \sR^{d_0} \to \sR^{d_1}$ and $\rho : \sR^{d_1} \to \sR^{d_{\mathrm{out}}}$ such that for every $\va = (a_1,\dots,a_n)$,
\begin{equation}
\label{eq:sum_decomposition_compressible}
\func(\va)
=
\rho\!\left(\sum_{i=1}^n \phi(\te(a_i,i))\right).
\end{equation}

Let $\bar{\va}$ be the prefix from the theorem statement, and let $\vb^1, \vb^2 \in A^k$ be suffixes such that
\begin{equation}
\func([\bar{\va}, \vb^1]) \neq \func([\bar{\va}, \vb^2]).
\end{equation}

It follows from equations \ref{eq:compressabillity_condition} and \ref{eq:sum_decomposition_compressible} that
\begin{equation}
\sum_{i=1}^k \phi(\te(b^{1}_i,i))
\;\neq\;
\sum_{i=1}^k \phi(\te(b^{2}_i,i)),
\end{equation}
where $\vb_1 = (b^{1}_1,\dots,b^{1}_k)$ and $\vb_2 = (b^{2}_1,\dots,b^{2}_k)$.

Define
\begin{equation}
\conv
=
\mathrm{conv}\!\left(
\left\{
\sum_{i=1}^k \phi(\te(b_i,i)) \;\middle|\; (b_1,\dots,b_k) \in A^k
\right\}
\right)
\;+\;
\sum_{i=1}^n \phi(\te(a_i,i)),
\end{equation}
and let $D$ denote the diameter of $\conv$. Since $\rho$ is continuous and $\conv$ is compact, the image $\rho(\conv)$ is also compact. Consequently, there exists $\vu_0 \in \sR^{d_{\mathrm{out}}}$ such that
\begin{equation}
\mathrm{dist}(\vu_0, \rho(\conv)) > C.
\end{equation}

We construct the transformer $\model$ similarly to the one constructed  in the proof of Lemma~\ref{lemma:compressible_main}, with two modifications to the second block.

First, we replace $\mW_O = \mI$ with
\begin{equation}
\mW_O=
\begin{pmatrix}
\mathbf{0} & \mathbf{0} & \mathbf{0} & \mathbf{0} \\
\mathbf{0} & \mathbf{0} & \mathbf{0} & \mathbf{0} \\
\mathbf{0} & \mathbf{0} & \mI & \mathbf{0} \\
\mathbf{0} & \mathbf{0} & \mathbf{0} & \mathbf{0}
\end{pmatrix}.
\end{equation}
Second, we modify the second feedforward function to approximate a continuous function $\bar{\rho}$ satisfying
\begin{equation}
\label{eq:non_compressable_ffn}
\bar{\rho}(\vw, \vx,\vy,\vz) =
\begin{cases}
\rho(i \cdot \vy),
& \text{if } \| \vx - \vp_i\| < \varepsilon'   \text{ and } i \cdot \vy \in \conv_{\frac{\delta}{2}} ,  \\[0.5em]
\vu_0,
& \text{if }  \| \vx - \vp_i\| < \varepsilon'   \text{ and } i \cdot \vy \notin \conv_{\delta},\\

\end{cases}
\end{equation}

where
\begin{equation}
\delta = \left(\frac{n+k}{n-1+k} - 1\right)\frac{D}{2},
\end{equation}
and
\begin{equation}
\conv_\varepsilon := \{\vz \in \sR^d : \mathrm{dist}(\vz,\conv) < \varepsilon\}.
\end{equation}

By the same reasoning as in Lemma~\ref{lemma:compressible_main}, Equation~\ref{eq:approximate_non_compressible} holds.

Let $\compress = (\cl_1, \cl_2)$ be any compression policy with compression budget $\length(n) < n$. Since the first block of $\model$ zeroes out all attention heads and depends only on the token embeddings, $\cl_1$ has no effect on the forward pass of $\model_{\compress,\va}(\vb)$. Thus, after the first block, the representations of the tokens $b_i$ are
\begin{equation}
\begin{bmatrix}
\tilde{\phi}(\te(b_1,1+n)) \\
\vdots \\
\tilde{\phi}(\te(b_k,k+n))
\end{bmatrix}.
\end{equation}

In the second block, since $\mW_Q = 0$, all attention scores are identical, and hence the attention weights are uniform, equal to $\frac{1}{k + \length(n)}$. It follows that the representation of each token $b_i$ after the second block is
\begin{equation}
\bar{\rho}\left( \phi(\te(\vb_i, i+n)), \vp_{i+n},
\frac{1}{\length(n)+k}\left(\tilde \vv + \sum_{j=1}^k \phi(\te(b_j,j+n))\right),
\mathbf{0}
\right),
\end{equation}
where the compressed prefix cache is given by
\begin{equation}
\cl_2(\mK_\va, \mV_\va) = (\tilde \mK_\va, \tilde \mV_\va) = ([\tilde \vk_1, \dots, \tilde \vk_{\length(n)}]^\top, [\tilde \vv_1, \dots, \tilde \vv_{\length(n)}]^\top)
\end{equation}
and
\begin{equation}
\tilde \vv =  \bigl(\sum_{i=1}^{\length(n)} \tilde \vv_i\bigr)_{2d_1: 3d_1}.
\end{equation}

Since $\frac{n+k}{\length(n)+k} >1$ by Lemma~\ref{lemma:convex_hull}, there exists a sequence $\bar{\vb}$ of length $k$ such that
\begin{equation}
\mathrm{dist}\!\left(
\frac{n+k}{\length(n)+k}\left(\vv + \sum_{j=1}^k \phi(\te(\bar{b}_j,j+n))\right),
\conv
\right)
>
\left(\frac{n+k}{\length(n)+k} - 1\right)\frac{D}{2}
\;\ge\;
\delta.
\end{equation}

Therefore, by Equation~\ref{eq:non_compressable_ffn}, the final representation of the token $\bar{b}_k$, which is the output of $\model_{\compress,\bar{\va}}(\bar{\vb})$, is equal to $\vu_0$.

On the other hand, by construction, $\model([\bar{\va}, \bar{\vb}]) \in \rho(\conv)$. Hence,
\begin{equation}
\|\model([\bar{\va}, \bar{\vb}]) - \model_{\compress,\bar{\va}}(\bar{\vb})\| > C,
\end{equation}
completing the proof.

\end{proof}

\begin{lemma}
\label{lemma:convex_hull}
Let $\vx_1,\dots,\vx_k \in \mathbb{R}^d$ be distinct points with $k\ge 2$, and let
\[
\conv := \operatorname{conv}\{\vx_1,\dots,\vx_k\}.
\]
Let $\alpha>1$. Then  for every $\vv \in \sR^d$ there exists $i\in [k]$ satisfying
\[
\operatorname{dist}(\alpha \vx_i+\vv,\conv)\ge \frac{\alpha-1}{2}\operatorname{diam}(\conv).
\]

\end{lemma}

\begin{proof}
Let
\[
D:=\operatorname{diam}(\conv)=\sup_{\vp,\vq\in \conv}\|\vp-\vq\|.
\]
Since the points $\vx_1,\dots,\vx_k$ are distinct and $k\ge 2$, we have $D>0$.

Choose $\vp,\vq\in \conv$ such that $\|\vp-\vq\|=D$ (such a pair exists since $\conv$ is compact), and define
\[
\vu:=\frac{\vp-\vq}{\|\vp-\vq\|}.
\]
The set $\{\langle \vu,\vx\rangle \mid \vx\in \conv \}$ forms a closed interval, and the width of $\conv$ in direction $\vu$ is
\[
\max_{\vx\in \conv}\langle \vu,\vx\rangle-\min_{\vx\in \conv}\langle \vu,\vx\rangle = D.
\]

For any $\vv \in \sR^d$, the width of $\alpha \conv+ \vv$ in direction $\vu$ is therefore
\[
\max_{\vx\in \alpha \conv+\vv}\langle \vu,\vx\rangle
-
\min_{\vx\in \alpha \conv+\vv}\langle \vu,\vx\rangle
=
\alpha D.
\]

Let
\[
\varepsilon := \frac{\alpha-1}{2}D.
\]
Suppose, for contradiction, that
\[
\operatorname{dist}(\alpha \vx_i+\vv,\conv)<\varepsilon
\qquad
\text{for all } i=1,\dots,k.
\]
Since
\[
\alpha \conv+\vv=\operatorname{conv}\{\alpha \vx_1+\vv,\dots,\alpha \vx_k+\vv\},
\]
and since the $\varepsilon$-neighborhood of $\conv$,
\[
\conv_\varepsilon:=\{\vz\in\mathbb{R}^d:\operatorname{dist}(\vz,\conv)<\varepsilon\},
\]
is convex, it follows that
\[
\alpha \conv+\vv \subseteq \conv_\varepsilon.
\]
Consequently, the width of $\alpha \conv+\vv$ in direction $\vu$ is strictly smaller than the width of $\conv_\varepsilon$ in direction $\vu$. This width is at most
\[
D+2\varepsilon.
\]
Therefore,
\[
\alpha D < D+2\varepsilon.
\]
But by the definition of $\varepsilon$,
\[
D+2\varepsilon
=
D+(\alpha-1)D
=
\alpha D,
\]
which is a contradiction. Hence there exists some $i\in\{1,\dots,k\}$ such that
\[
\operatorname{dist}(\alpha \vx_i+\vv,\conv)\ge \varepsilon.
\]
\end{proof}

\textbf{Note.} We believe that a finer-grained analysis of transformer compressibility may be possible through the theory of Chebyshev systems~\citep{karlin1966tchebycheff}, a classical framework with broad applications across both pure and applied mathematics \citep{karlin1966optimal,micchelli1977moment,eitan2021centered,karlin1966chebyshevian}. Chebyshev systems provide tools for controlling the zeros, sign changes, and interpolation structure of linear combinations of functions, which makes them a natural candidate for studying the degrees of freedom available in attention-like mixtures of value functions. We leave a systematic development of this connection to future work.

%% file: appendix/method.tex
\section{\ourmethod implementation details}\label{app:method}

\paragraph{Masked attention forward pass.}
Instead, at each layer group, the train-time KV sparsification policy defines which previous KV slots are visible to attention. For layer $\ell$, let $m_j^{(\ell)}$ be the most recent mask produced by the sparsification policy. The attention for query token $t$ is evaluated over
\begin{equation}
    \mathcal{A}_{t}^{\ell} =\{j\le t : m_j^{(\ell)}=1\}
\end{equation}
Thus all valid queries are still processed by the transformer, but their access to past KV slots is restricted to the active, unmasked KV slots. The masks are independent across compression points: a token dropped in one layer may be selected again in a later layer. No masking is used at evaluation time in our comparisons.

\paragraph{Router parameterization.}
For the main experiments in the paper, we use lightweight learnable router modules, parameterized as linear attention, as our KV sparsification policy. For a sequence $x_{1:T}$, let $h_t^\ell \in \sR^d$ denote the hidden state of token $t$ before decoder layer $\ell$. We insert routers at a small set of compression layers $\gC$ (e.g., for \textsc{Qwen2.5-0.5B} we use $\gC=\{1,7,13,19\}$). Each router independently predicts which KV slots are masked for the following group of layers. Given hidden states $H^\ell=(h_1^\ell,\ldots,h_T^\ell)$
and, the router first computes
\begin{equation}
    \tilde h_t = \mathrm{LN}(h_t^\ell),\qquad
    q_t = \phi(W_Q \tilde h_t),\quad
    k_t = \phi(W_K \tilde h_t),\quad
    r_t = W_V \tilde h_t,
    \qquad \phi(z)=\mathrm{ELU}(z)+1 .
\end{equation}
It then forms a causal linear-attention summary without storing a router KV
cache,
\begin{equation}
    S_t=\sum_{j\le t} k_j r_j^\top,\qquad
    z_t=\sum_{j\le t} k_j,\qquad
    a_t = W_O \frac{q_t^\top S_t}{q_t^\top z_t+\epsilon}.
\end{equation}
The keep probability is computed from a cosine score between a pointwise
projection of the current state and the state after the router summary:
\begin{equation}
    u_t=\frac{W_P h_t^\ell}{\|W_P h_t^\ell\|_2},\qquad
    w_t=\frac{h_t^\ell+\alpha a_t}{\|h_t^\ell+\alpha a_t\|_2},\qquad
    p_t=\frac{1-\langle u_t,w_t\rangle}{2}.
\end{equation}
The binary routing decision is
\begin{equation}
    m_t^\ell = \mathbf{1}\{p_t>\tau\},
\end{equation}
with $\tau=0.5$ in all reported runs. During training, this hard threshold is optimized with a straight-through estimator: the forward pass uses the binary mask above, while the backward pass passes gradients through the pre-threshold keep probability
$p_t$ to the router parameters. We initialize $W_P=-I$ and $\alpha=0$, which gives $p_t=1$ for every valid token, so training starts from the dense model.

\paragraph{Training objective.}
Let $p_\theta^{\mathrm{mask}}(\cdot\mid x_{<t})$ denote the masked forward pass and $p_\theta^{\mathrm{dense}}(\cdot\mid x_{<t})$ the same model with masking disabled. We optimize
\begin{equation}
    \begin{aligned}
    \mathcal{L}
    ={}&
    \lambda_{\mathrm{mask}}\frac{1}{T-1}\sum_{t=2}^{T}
    \KL\!\left(
    \operatorname{sg}\!\left[p_\theta^{\mathrm{dense}}(\cdot\mid x_{<t})\right]
    \,\middle\|\,
    p_\theta^{\mathrm{mask}}(\cdot\mid x_{<t})
    \right)
    \\
    &+
    \lambda_{\mathrm{anchor}}\frac{1}{T-1}\sum_{t=2}^{T}
    -\log p_\theta^{\mathrm{dense}}(x_t\mid x_{<t})
    +
    \lambda_{\mathrm{budget}}\frac{1}{|\mathcal{C}|}
    \sum_{c\in\mathcal{C}} \mathcal{B}_c .
    \end{aligned}
\end{equation}
where $\operatorname{sg}$ stops gradients through the dense teacher. When using the router based KV sparsification policy, we use $\lambda_\mathrm{budget} \neq 0$ and the budget term follows the load-balancing objective from~\citet{hwang2025hnet}. For
target keep rate $\rho$, define
\begin{equation}
    F_c=\frac{1}{T}\sum_{t=1}^{T} m_t^c,
    \qquad
    G_c=\frac{1}{T}\sum_{t=1}^{T} p_t^c,
    \qquad
    \mathcal{B}_c
    =
    \frac{F_cG_c}{\rho}
    +
    \frac{(1-F_c)(1-G_c)}{1-\rho}.
\end{equation}
All transformer and router parameters are updated during continued pretraining; in the reported runs $\rho=0.5$, $\lambda_{\mathrm{mask}}=1$,
$\lambda_{\mathrm{anchor}}=1$, and $\lambda_{\mathrm{budget}}=0.1$.


%% file: appendix/extended_experimental_details.tex
\section{Extended Experimental Details}\label{app:exp_details}
This section provides additional details for the empirical evaluation described in Section~\ref{sec:experiments}. We describe the continued-pretraining setup, including the data, architecture hyperparameters, and compression objective, and then give additional details on the post-hoc compression evaluations and experimental setups used throughout the paper. Unless otherwise stated, evaluations compare the base model to the \ourmethod checkpoint with masking disabled. All experiments were run on up to 8 H100 GPUs.

\begin{table*}[h]
  \centering
  \small
  \caption{\textbf{\ourmethod continued-pretraining hyperparameters for \textsc{Qwen2.5} checkpoints.}}
  \label{tab:training-hparams}
  \begin{tabular}{@{}p{0.31\textwidth}p{0.31\textwidth}p{0.31\textwidth}@{}}
    \toprule
    Setting & \textsc{Qwen2.5-0.5B} & \textsc{Qwen2.5-1.5B} \\
    \midrule
    Base checkpoint & \path{Qwen/Qwen2.5-0.5B} & \path{Qwen/Qwen2.5-1.5B} \\
    Training data & FineWeb-Edu train, streaming & FineWeb-Edu train, streaming \\
    Sequence length & 1024 tokens & 1024 tokens \\
    Compression layers & 0, 6, 12, 18 & 0, 7, 14, 21 \\
    Router feature dimensions & 64 & 64 \\
    Train threshold $\tau_{\mathrm{train}}$ & 0.5 & 0.5 \\
    Target keep rate & 50\% & 50\% \\
    Compression loss weight $\lambda_{\mathrm{mask}}$ & 1 & 1 \\
    Dense anchor weight $\lambda_{\mathrm{anchor}}$ & 1 & 1 \\
    Budget weight $\lambda_{\mathrm{budget}}$ & 0.1 & 0.1 \\
    Optimizer & AdamW & AdamW \\
    Weight decay & 0.01 & 0.01 \\
    Gradient clipping & 1.0 max norm & 1.0 max norm \\
    Peak learning rate & $10^{-4}$ & $10^{-4}$ \\
    Minimum learning rate & $5\times10^{-6}$ & $5\times10^{-6}$ \\
    Warmup steps & 600 & 600 \\
    Tokens per optimizer step & 131{,}072 & 131{,}072 \\
    Training hardware & 8$\times$H100 GPUs & 8$\times$H100 GPUs \\
    Training steps & 40k & 40k \\
    Training tokens & $5.24\times10^9$ & $5.24\times10^9$ \\
    \bottomrule
  \end{tabular}
\end{table*}
\subsection{Continued pretraining runs}\label{app:cpt}
We train two model sizes, initialized from the public \textsc{Qwen2.5-0.5B} and
\textsc{Qwen2.5-1.5B}~\citep{qwen2.5} checkpoints, using continued pretraining on FineWeb-Edu~\citep{penedo2024fineweb} at a context length of 1024. The compression-aware objective is the one described in the method section: a compressed self-distillation loss matched to the model's dense distribution, an uncompressed next-token prediction anchor, and a budget penalty. All transformer and compression module parameters are updated.

The router layers are inserted at layers 0, 6, 12, and 18 for \textsc{Qwen2.5-0.5B} and layers 0, 7, 14, and 21 for \textsc{Qwen2.5-1.5B}. Each module is a causal linear-attention boundary predictor (as described in Appendix~\ref{app:method}) with 64 feature dimensions, a threshold $\tau=0.5$, and a target keep rate of 50\%. The objective weights are $\lambda_{\mathrm{mask}}=1$, $\lambda_{\mathrm{anchor}}=1$, and $\lambda_{\mathrm{budget}}=0.1$. We use AdamW with peak learning rate $10^{-4}$, 600 warmup steps, minimum learning rate $5\times 10^{-6}$, weight decay 0.01, and gradient clipping at norm 1.0. Both model sizes use a batch of 131{,}072 tokens per optimizer step, corresponding to 128 sequences of length 1024, and are trained on 8 GPUs. Runs are configured for 40k optimizer steps, which see $5.24\times10^9$ tokens.

At evaluation time, the boundary predictors are turned off, and the post-hoc compressor named in each experiment is the only cache-reduction procedure whose quality is measured. Figure~\ref{fig:qwen0p5-training-diagnostics} reports the compressed and uncompressed validation NTP loss as well as the KV retention ratio throughout the \ourmethod continued pretraining for \textsc{Qwen2.5-0.5B}. Table~\ref{tab:training-hparams} reports all hyperparameters and training setup for both runs.

\begin{figure}[t]
  \centering
  \begin{tabular}{@{}c@{\hspace{0.01\linewidth}}c@{}}
    \subfigure[Dense NTP loss]{
      \includegraphics[width=0.48\linewidth]{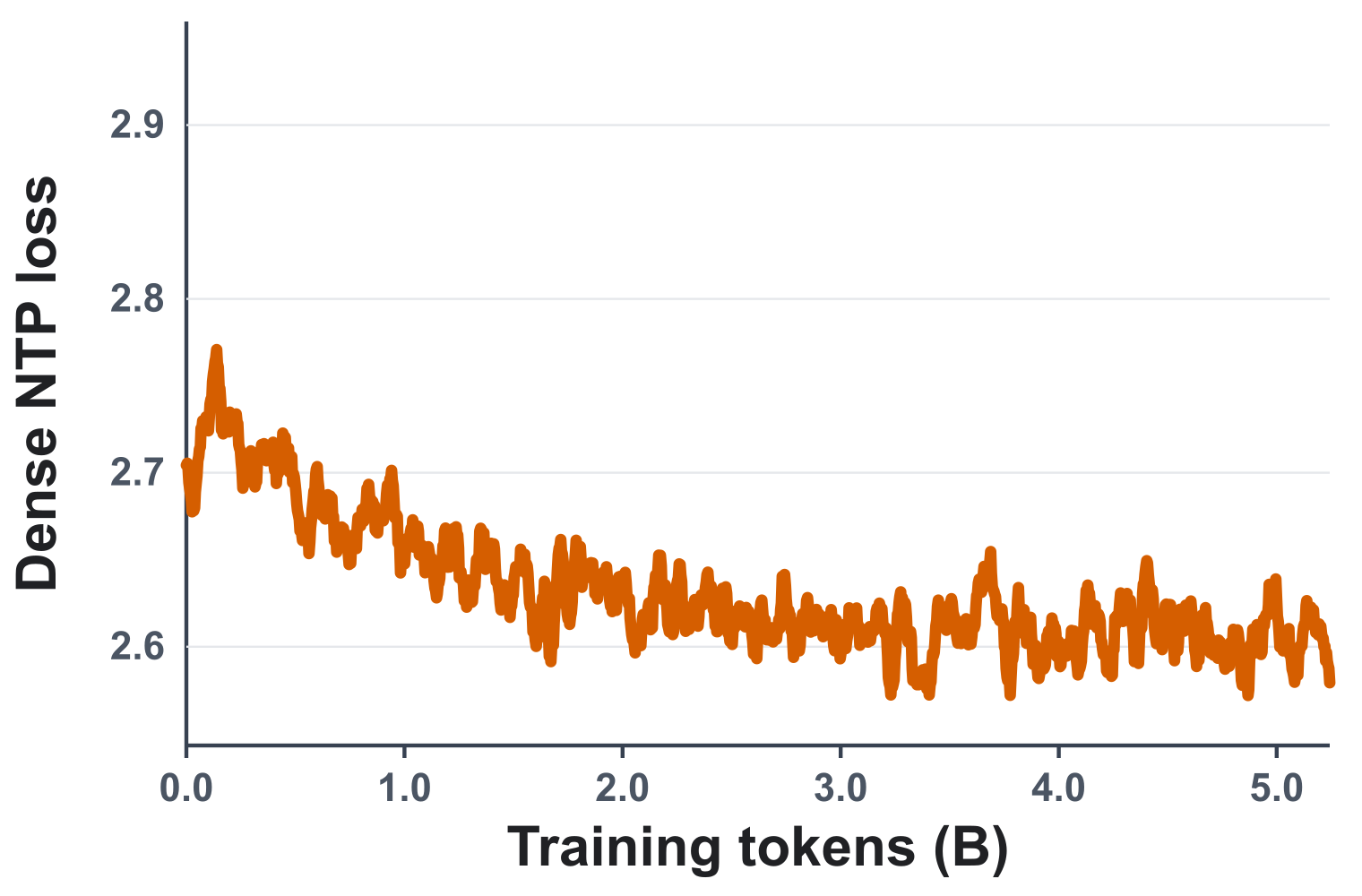}
    } &
    \subfigure[Compressed NTP loss]{
      \includegraphics[width=0.48\linewidth]{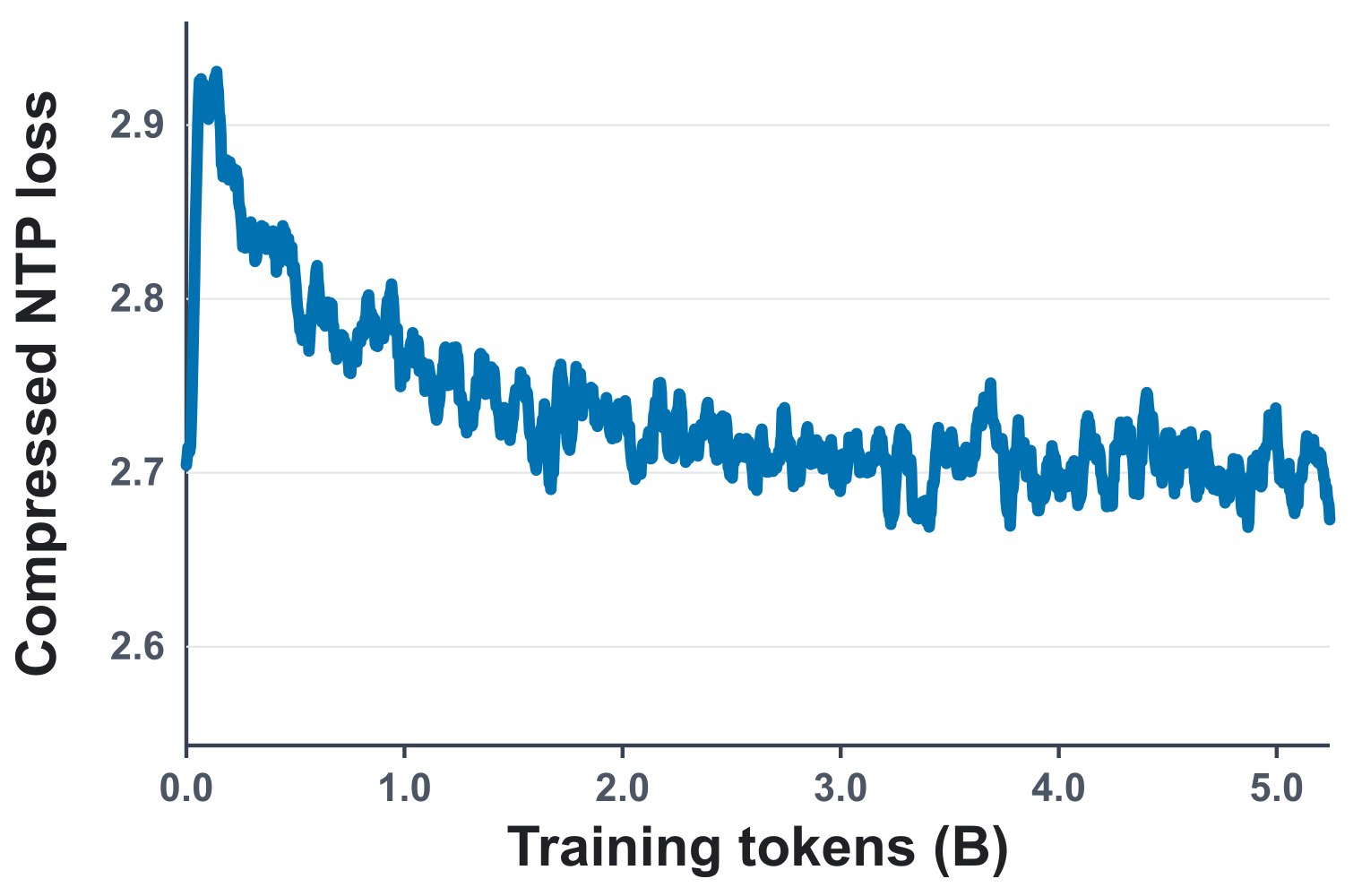}
    } \\
    \multicolumn{2}{c}{
      \subfigure[KV retention ratio]{
        \includegraphics[width=0.50\linewidth]{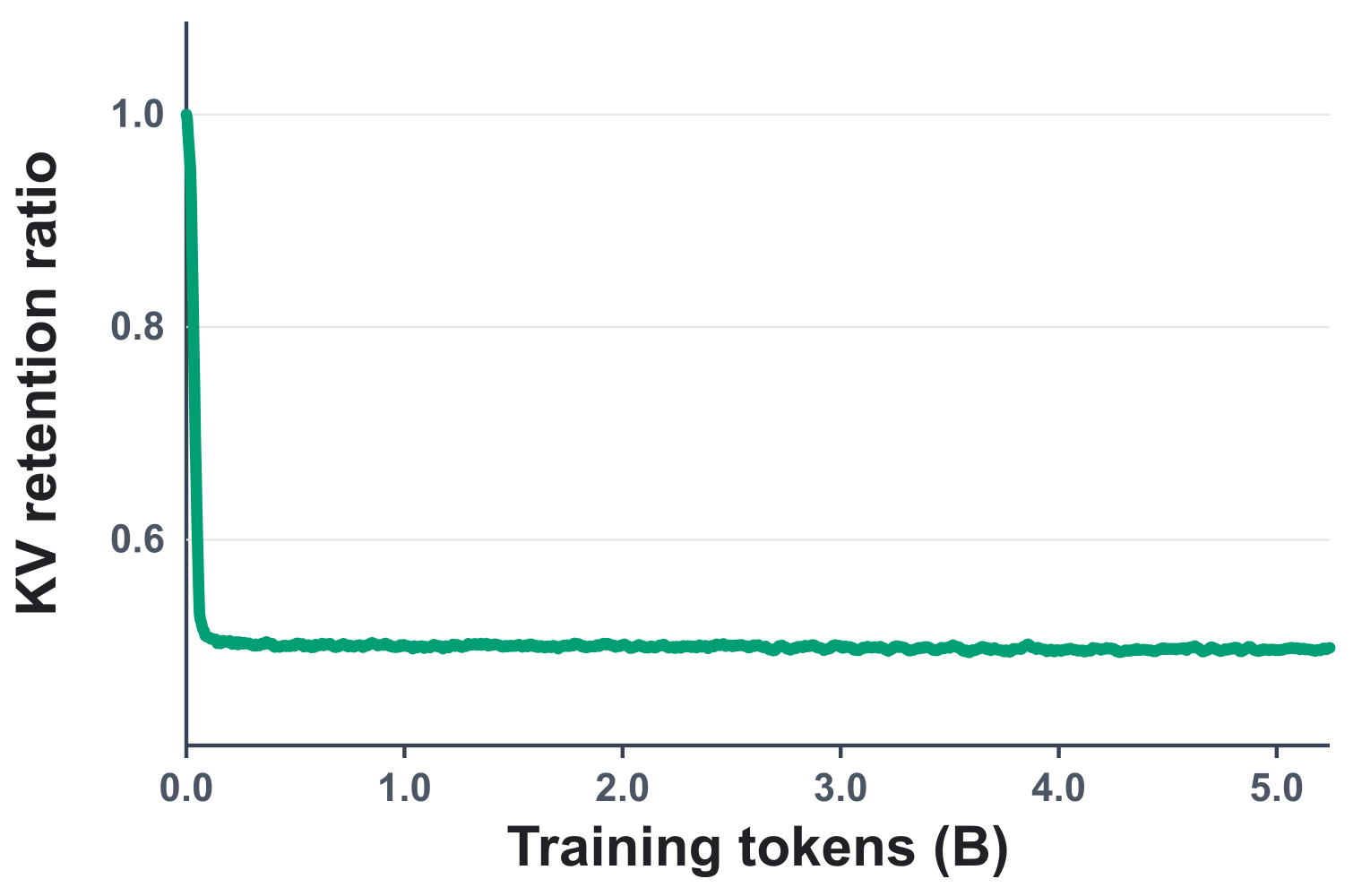}
      }
    }
  \end{tabular}
  \caption{\textbf{\textsc{Qwen2.5-0.5B} \ourmethod continued pretraining run.} We plot dense validation next-token-prediction loss, compressed-path validation next-token-prediction loss, and the realized KV retention ratio over the 5.24B-token \ourmethod continued-pretraining run. The retention trace reports the fraction of KV slots kept by the learned routers, whose budget target is 50\%.}
  \label{fig:qwen0p5-training-diagnostics}
\end{figure}

\subsection{No compression QA evaluation}
To check that our training procedure does not substantially degrade ordinary model behavior, we evaluate normalized multiple-choice accuracy on a variety of question answering benchmarks using the LightEval harness~\citep{habib2023lighteval}. The evaluation suite contains HellaSwag~\citep{zellers2019hellaswag}, WinoGrande~\citep{sakaguchi2020winogrande}, PIQA~\citep{bisk2020piqa}, Social IQa~\citep{sap2019socialiqa}, OpenBookQA~\citep{mihaylov2018openbookqa}, ARC-Easy and ARC-Challenge~\citep{clark2018arc}. We use 1000 validation examples per task, except for OpenBookQA, whose validation split contains 500 examples in this setup. The reported average is the unweighted mean over these eight tasks. The purpose of this evaluation is to measure whether the trained checkpoints maintain language modeling performance on standard short-context multiple-choice tasks before we test their behavior under hoc KV cache compression.

\subsection{Suffix perplexity under prefix KV cache compression}\label{app:suffix}
The suffix-prediction experiments evaluate whether the continued-pretrained checkpoints are easier targets for post-hoc prefix KV cache compression. The in-domain evaluation data consists of held-out blocks from FineWeb. Each evaluation example is a 1024-token excerpt split into a 768-token prefix and a 256-token suffix. We first run the model on the full 1024-token sequence and record the suffix logits and per-layer attention traces. The first 768 tokens define the prefix cache to be compressed. The following 256 suffix tokens are never removed or compressed: after a compact prefix cache has been constructed, the same suffix tokens are appended normally and scored under the original transformer layers. This protocol is shared by the Attention Matching table and the gradient based optimization step-curve figures, but the two experiments fit the compact cache differently.

\paragraph{Attention-Matching.}
We use the AM-HighestAttnKeys Attention matching variant from~\citet{zweiger2026fastkv}. For a keep ratio $\rho \in \{0.05, 0.1, 0.2, 0.4\}$, Attention Matching constructs a compact prefix cache independently for each layer and key-value head.  Let $K,V \in \mathbb{R}^{L \times d}$ denote the dense prefix keys and values for one key-value head, where $L=768$, and let $Q \in \mathbb{R}^{M \times d}$ denote suffix query states from the corresponding grouped query heads. We subsample at most 256 suffix queries per key-value head. Attention Matching first selects $\lceil \rho L \rceil$ source keys with the largest root-mean-square attention weight over the sampled suffix queries. The selected keys form the compact keys $C_1$. A scalar log-bias vector $\beta$ is then fit by two iterations of box-constrained nonnegative least squares so that the compact cache approximately matches the dense attention normalizer. Finally, compact values $C_2$ are fit by ridge least squares to reconstruct the dense attention outputs,
\[
  \operatorname{softmax}\!\left(\frac{Q C_1^\top}{\sqrt{d_k}} + \beta\right) C_2
  \approx
  \operatorname{softmax}\!\left(\frac{Q K^\top}{\sqrt{d_k}}\right) V .
\]
We use ridge coefficient $10^{-4}$ with spectral scaling and solve the value regression with a least-squares solver. After constructing the compact prefix cache, we run the model on the same prefix-suffix sequence while replacing the dense prefix cache with the compact cache for all subsequent suffix positions. We report  the increase in suffix perplexity relative to the same model's dense-prefix forward pass, KL divergence to the native dense-prefix logits, and top-1 agreement with the native logits.  All metrics are averaged over 128 evaluation examples.

\paragraph{Gradient-based KV cache compression.}
In this setting, we directly optimize a small set of continuous key and value vectors for each example while keeping all model parameters fixed.  The resulting optimization curves measure how much information about the prefix can be preserved in a compact KV cache when the cache itself is allowed to adapt to the example. The dense model is first run on the full prefix-suffix sequence to produce teacher logits on the suffix. For a target keep ratio
\[
  \rho \in \{0.05,0.1,0.15,0.2,0.25,0.3,0.35,0.4,0.5\},
\]
we allocate a compact prefix cache with
\[
  m = \left\lceil \rho L \right\rceil
\]
KV slots per layer and key-value head, where \(L=768\) is the dense prefix length. The compact cache consists of learnable key and
value tensors
\[
  \widetilde{K}_{\ell,h}, \widetilde{V}_{\ell,h} \in \mathbb{R}^{m \times d}
\]
for each layer \(\ell\) and key-value head \(h\).   We initialize \(\widetilde{K} _{\ell,h}\) and \(\widetilde{V}_{\ell,h}\) from the first \(m\) dense prefix KV states for that layer and head. The compact cache is optimized per example using the suffix teacher distribution from the dense full-prefix forward pass. Let \(z_t^\star\) denote the dense teacher logits at suffix position \(t\), and let \(z_t(\widetilde{K},\widetilde{V})\) denote the logits obtained when the dense prefix cache is replaced by the compact cache.  We minimize
\[
  \mathcal{L}_{\mathrm{KV}}
  =
  \frac{1}{|\mathcal{S}|}
  \sum_{t \in \mathcal{S}}
  D_{\mathrm{KL}}\!\left(
      \operatorname{softmax}(z_t^\star)
      \,\middle\|\,
      \operatorname{softmax}(z_t(\widetilde{K},\widetilde{V}))
  \right),
\]
where \(\mathcal{S}\) is the set of suffix prediction positions.  The optimization updates only the compact KV tensors; model weights
are never updated.  We use AdamW for 100 gradient steps with learning rate \(10^{-2}\), weight decay \(0\), and gradient clipping at
norm \(1.0\).  We record intermediate compact caches after steps
\[
  \{0,1,2,5,10,20,50,100\}
\]
to produce the optimization-step curves. After each recorded optimization step, we evaluate next-token prediction on the same suffix using the compact cache in place of the dense prefix cache. The reported perplexity is computed against the ground-truth suffix tokens. 

This evaluation should be interpreted as a controlled probe of the cache-construction problem in post-hoc compression. For each example, the compact prefix cache is fit using supervision derived from the same prefix--suffix sequence on which it is evaluated: Attention Matching uses suffix attention/query traces to construct the compact cache, while the gradient-based method uses suffix teacher logits. Thus, the comparison asks whether the trained checkpoints make the per-example compact-cache optimization problem easier under matched cache budgets and matched supervision, rather than testing transfer from a cache fit on one suffix to unseen suffixes (which is tested in the next experiments).

\subsection{Needle-in-a-haystack retrieval under KV cache compression} \label{app:niah-kv-compression}
We evaluate retrieval under prefix KV cache compression using a deterministic needle-in-a-haystack task. Each example contains neutral filler text, passkey-like distractor numbers, a single six-digit passkey, and a final query asking for the passkey. The prompt length before the final query is 1024 tokens. The model is evaluated using the compressed prompt and the query.

\paragraph{Example construction.}
To generate each example, we sample a six-digit string \(a \in \{\texttt{000000},\ldots,\texttt{999999}\}\), which defines the passkey stored in the prefix and the answer expected at evaluation. For example, if \(a=\texttt{483920}\), the needle inserted into the haystack is
\[
\texttt{Memory record: special\_passkey=483920.}
\]
The final query is
\[
\texttt{Memory record: special\_passkey=}
\]
and the answer is exactly
\[
\texttt{483920}.
\]
The passkey is inserted at relative depths
\[
  d \in \{0,0.25,0.5,0.75,1\},
\]
where \(d=0.0\) places the needle at the beginning of the haystack and \(d=1\) places it at the end. In addition to the answer passkey, examples contain distractors: passkey-like six-digit strings that are explicitly not equal to the answer. These distractors are inserted into the sequence so that the model cannot solve the task by copying an arbitrary number. For example, if the answer
is \(\texttt{483920}\), a possible distractor sentences is
or
\[
\texttt{Ignore the misleading special passkey candidate 672041.}
\]
During example construction, the haystack is assembled from filler text before and after the needle. At each filler-sentence draw, with probability \(0.35\) we insert a distractor sentence; otherwise we insert a neutral
sentence unrelated to the retrieval key. Distractor codes are sampled uniformly from the six-digit code space and rejected if they match the answer. The needle is then inserted at the requested relative depth in the haystack, and the query is appended after the haystack. For each compression ratio, we evaluate 20 instances per depth totaling at a 100 examples.

\paragraph{Training the compact KV cache.}
For each example, we train a new compact prefix KV cache while keeping all model weights frozen. Let \texttt{<haystack>} denote the haystack prefix, including the filler text, needle and any distractors, but excluding the final query and answer. We construct an
reconstruction sequence
\[
    \texttt{<haystack><instruction><haystack>},
\]
where the instruction is
\[
    \texttt{Reproduce the preceding passage verbatim.} 
\]
\[
    \texttt{Preserve every word, digit, and punctuation mark.}
\]
The compact cache replaces the KV states for the first \(\texttt{<haystack>}\). The loss is applied only on the second \(\texttt{<haystack>}\), so the optimizer must store enough information in the compact cache to reconstruct the original prefix, including the passkey and distractors.

For a keep ratio \(\rho\), the compact cache contains
\[
m = \lceil \rho L \rceil
\]
KV slots per layer and KV head, where \(L\) is the haystack length. The compact keys and values are initialized by selecting $m$ random haystack KV states and then optimized directly with AdamW. As in Appendix~\ref{app:suffix}, we use a KL objective between the suffix logits (computed on the second copy of \texttt{<haystack>}) produced by the dense model and the ones produced when attending to the compact cache. For each example, this objective is optimized by running AdamW for 100 steps with learning rate \(10^{-2}\), weight decay \(0\), and gradient clipping at norm \(1.0\). 

\paragraph{Evaluation.}
At evaluation time, the haystack prefix is represented by the optimized compact KV cache, while the final query is appended normally and remains uncompressed. We then perform greedy answer generation and compare the generated completion against the six-digit passkey. In the table, we report native-success-conditioned exact-match retrieval accuracy: for each model, keep ratio, and compact-cache optimization step, we report the percentage of those examples for which the compressed-prefix model also generates exactly the passkey.


\subsection{Long-form question answering under KV cache compression} \label{app:lbv2-cache-reconstruction}
We evaluate long-context question answering under post-hoc KV cache compression on LongBench v2. Because the compression procedure optimizes a separate compact KV cache for each prompt, this evaluation is substantially more expensive than ordinary decoding. We therefore report results on a fixed subset of seven LongBench v2 subdomains: Academic, Agent history QA, Knowledge graph reasoning, Legal, Many-shot learning, New language translation, and Table QA. This subset contains 221 examples in total.

For each LongBench example, we construct a multiple-choice prompt from the context, question, and four answer choices. The prefix is
\[
\texttt{Please read the following text and answer the question below.}
\]
followed by the context from the benchmark. The suffix is
\[
\texttt{What is the correct answer to this question: <question>}
\]
followed by the four choices \texttt{(A)}, \texttt{(B)}, \texttt{(C)}, and \texttt{(D)}, and the answer-format prefix
\[
\texttt{The correct answer is (}.
\]
We score the four continuations \texttt{A)}, \texttt{B)}, \texttt{C)}, and \texttt{D)} by total log-likelihood and predict the answer with the highest score. No free-form generation is used.

For each example, model, and keep ratio $\in \{0.1, 0.2, 0.5\}$, we optimize a compact KV cache for example's context while keeping all model weights frozen. The optimization uses a reconstruction sequence as in Appendix~\ref{app:niah-kv-compression}. As common in LongBench evaluations, if the reconstruction sequence length is larger than the context length of the model (32{,}768 tokens) we truncate the context using middle truncation (keeping a prefix and suffix of equal length)~\citep{bai2025longbench}. The KC objective is a KL divergence between the dense model's next-token distribution and the compact-cache model's next-token distribution on all tokens of the repeated prefix. The compressed caches are initialized from the first prefix KV slots. We optimize each compact cache for 100 AdamW steps with learning rate $10^{-2}$, weight decay 0, gradient clipping norm 1.0, and gradient checkpointing enabled. At evaluation time, the optimized compact KV cache replaces the full KV cache of the context prefix. The question, answer choices, and answer prefix are then processed normally, without compression. We report accuracy, grouped by subdomain and averaged over the 221 examples.

%% file: appendix/additional_experimental_results.tex
\section{Additional Experimental Results}\label{app:add_exp}
This section collects supplementary results that are complementary to Section~\ref{sec:experiments}. In Section~\ref{app:method-ablation}, we ablate the choice of the train-time sparsification policy used in \ourmethod. In Section~\ref{app:cross-domain-am}, we test whether the Attention Matching performance gains we observe in Section~\ref{sec:experiments} generalize to prefix--suffix from additional text corpora.

\subsection{Train-time KV sparsification policy ablation}\label{app:method-ablation}
We use this ablation to choose the train-time KV sparsification policy used in the main experiments. Starting from \textsc{Qwen2.5-0.5B}, we train three checkpoints with the same self-distillation and dense-anchor losses and the same 50\% target keep rate, varying only the policy that selects active KV slot during the masked forward pass. \textsc{Rand} uniformly samples an exact-count 50\% subset of valid KV slots at each sparsification point. \textsc{Attn} is an H2O-inspired~\citep{zhang2023h2o} attention-mass baseline: it computes dense causal attention at the routed layer, sums the attention mass received by each source token over heads and valid query positions, and keeps the top 50\% source KV slots. \textsc{Router} uses the lightweight learned linear-attention routers described in Appendix~\ref{app:method}. Because \textsc{Rand} and \textsc{Attn} satisfy the target keep rate by construction, we set $\lambda_\mathrm{budget}=0$ for those runs; for \textsc{Router}, we use the budget regularizer described in Appendix~\ref{app:method}. In all comparisons below, the trained checkpoints are evaluated in the unmasked forward pass mode, with the training-time sparsification mechanism disabled, and post-hoc compression is applied afterward.

\begin{table*}[t]
    \centering
    \small
    \setlength{\tabcolsep}{4pt}
    \caption{\ourmethod implemented with a \textsc{Router} policy retains base model performance better than the fixed \textsc{Rand} and \textsc{Attn} policies.}\label{tab:ab-reasoning}
    \vspace{10pt}
    \resizebox{\linewidth}{!}{
        \begin{tabular}{llcccccccc}
        \toprule
        Model & Variant & HellaSwag & WinoGrande & PIQA & OpenBookQA & ARC-E &
        ARC-C & Avg. \\
        \midrule
        \multirow{4}{*}{\textsc{Qwen2.5-0.5B}}
          & Base & \textbf{54.6} & 52.9 & 70.4 & \textbf{38.2} & 60.8
        & 32.2 & 51.5  \\
          & \textsc{Router} & 53.6 & \textbf{54.1} & \textbf{72.1} & 37.4 &
        \textbf{63.6} & \textbf{32.4} & \textbf{52.2}  \\
          & \textsc{Rand}& 52.6 & 52.9 & 70.2 & 35.4 & 61.8 & 31.1 & 50.7 \\
          & \textsc{Attn} & 50.7 & 53.6 & 70.0 & 35.6 & 60.2 & 30.9 & 50.2 \\
        \bottomrule
        \end{tabular}
    }
    \vspace{-3mm}
\end{table*}
\paragraph{Uncompressed performance.}
\textsc{Router} is the only sparsification policy in this ablation that preserves dense downstream accuracy at or above the base model average. Its average score is 52.2, compared to 51.5 for the base model, 50.7 for \textsc{Rand}, and 50.2 for \textsc{Attn}. The fixed policies still retain much of the base model's zero-shot performance after continued pretraining.

\paragraph{Attention Matching Evaluation.}
All checkpoints are evaluated with the same Attention Matching procedure described in Appendix~\ref{app:suffix}. Table~\ref{tab:am_random_vs_ours} shows that all three train-time sparsification policies Attention Matching performance compared to the original base checkpoint for every reported metric and budget. This further supports the \ourmethod framework, demonstrating that the downstream compression improvements are not tied exclusively to learned router parameterization. Among the policies, \textsc{Router} gives the best result for every reported keep ratio and metric. The margin over \textsc{Rand} is sometimes small, especially at 10\% and 40\% keep ratio, but \textsc{Router} is consistently at least as good under compression and is clearly stronger on dense task retention. We therefore use \textsc{Router} in the main experiments as the default train-time sparsification policy.

\begin{table}[h]
    \centering
    \caption{
    Attention-matching prefix compression on FineWeb-Edu for the base model and three \ourmethod models with the \textsc{Router}, \textsc{Rand}, and \textsc{Attn} KV sparsification policies. We report degradation relative to each model's native dense-prefix forward pass. Lower $\Delta$PPL and KL are better; higher top-
  1 agreement is better.
    }
    \label{tab:am_random_vs_ours}
    \resizebox{\linewidth}{!}{
        \begin{tabular}{c cccc cccc cccc}
        \toprule
        Keep ratio
        & \multicolumn{4}{c}{$\Delta$PPL ($\downarrow$)}
        & \multicolumn{4}{c}{KL ($\downarrow$)}
        & \multicolumn{4}{c}{Top-1 ($\uparrow$)} \\
        \cmidrule(lr){2-5}
        \cmidrule(lr){6-9}
        \cmidrule(lr){10-13}
        & Base & \textsc{Router} & \textsc{Rand} & \textsc{Attn}
        & Base & \textsc{Router} & \textsc{Rand} & \textsc{Attn}
        & Base & \textsc{Router} & \textsc{Rand} & \textsc{Attn} \\
        \midrule
        0.05 & 0.780 & \textbf{0.461} & 0.547 & 0.593 & 0.0562 & \textbf{0.0385} & 0.0420 & 0.0440 & 88.5\% & \textbf{91.1\%} & 90.3\% & 90.3\% \\
        0.10 & 0.662 & \textbf{0.422} & 0.427 & 0.495 & 0.0483 & \textbf{0.0321} & 0.0356 & 0.0363 & 89.3\% & \textbf{91.4\%} & 90.9\% & 90.8\% \\
        0.20 & 0.777 & \textbf{0.438} & 0.547 & 0.594 & 0.0549 & \textbf{0.0374} & 0.0409 & 0.0396 & 88.7\% & \textbf{90.8\%} & 90.4\% & 90.5\% \\
        0.40 & 0.592 & \textbf{0.360} & 0.361 & 0.386 & 0.0431 & \textbf{0.0272} & 0.0283 & 0.0288 & 90.3\% & \textbf{92.5\%} & 92.3\% & 92.4\% \\
        \bottomrule
        \end{tabular}
    }
\end{table}

\subsection{Cross-domain attention matching evaluation}\label{app:cross-domain-am}
We test whether the Attention Matching optimization gains observed for \ourmethod transfer to prefix--suffix pair from additional textual domains. We evaluate Attention Matching KV cache compression on validation examples from three held-out corpora: WikiText-103~\citep{merity2016pointer}, PG-19~\citep{rae2019compressive}, and arXiv~\citep{cohan2018discourse}. For each example, we take a 768-token prefix and a 256-token suffix, compress only the prefix KV cache with the same Attention Matching procedure used in Appendix~\ref{app:suffix}, and then score suffix next-token prediction against each model's own dense full-prefix distribution.

Across all three held-out domains, the compression-aware checkpoint remains a better target for the same post-hoc compressor. It reduces KL and increases top-1 agreement for every corpus and keep ratio. It also lowers $\Delta$PPL in 11 of 12 settings; the only exception is arXiv at 40\% keep ratio, where $\Delta$PPL increases even though KL and top-1 agreement still improve. We therefore interpret this table as evidence that \ourmethod improves cross-domain compressibility under Attention Matching.

\begin{table*}[H]
    \centering
    \small
    \setlength{\tabcolsep}{4pt}
    \caption{\textbf{Attention Matching cross-domain completion under prefix-cache compression.}
    We compare the base \textsc{Qwen2.5-1.5B} model and the \ourmethod-trained \textsc{Qwen2.5-1.5B} checkpoint, on held-out prefix--suffix pair from different domains. Each example contains a 768-token prefix and a 256-token suffix; Attention Matching compresses only the prefix KV cache. Metrics are computed relative to each model's native dense full-prefix suffix distribution, so lower $\Delta$PPL and KL and higher top-1 agreement are better.}
    \vspace{6pt}
    \label{tab:cross-domain-am}
    \begin{tabular}{@{}ccc|ccc@{}}
    \hline
    Corpus & Keep ratio & Variant & $\Delta$PPL ($\downarrow$) & KL ($\downarrow$) & Top-1 ($\uparrow$) \\
    \hline
    \multirow{8}{*}{arXiv}
     & \multirow{2}{*}{5\%} & Base & 1.659 & 0.1875 & 83.4\% \\
     &  & \ourmethod & \textbf{1.530} & \textbf{0.1644} & \textbf{84.5\%} \\
     & \multirow{2}{*}{10\%} & Base & 1.292 & 0.1580 & 83.8\% \\
     &  & \ourmethod & \textbf{1.176} & \textbf{0.1340} & \textbf{85.6\%} \\
     & \multirow{2}{*}{20\%} & Base & 1.618 & 0.1836 & 83.5\% \\
     &  & \ourmethod & \textbf{1.243} & \textbf{0.1488} & \textbf{84.7\%} \\
     & \multirow{2}{*}{40\%} & Base & \textbf{2.002} & 0.1907 & 84.5\% \\
     &  & \ourmethod & 2.842 & \textbf{0.1599} & \textbf{86.1\%} \\
    \hline
    \multirow{8}{*}{PG-19}
     & \multirow{2}{*}{5\%} & Base & 0.735 & 0.0609 & 88.2\% \\
     &  & \ourmethod & \textbf{0.514} & \textbf{0.0371} & \textbf{89.6\%} \\
     & \multirow{2}{*}{10\%} & Base & 1.508 & 0.0913 & 86.8\% \\
     &  & \ourmethod & \textbf{0.736} & \textbf{0.0457} & \textbf{88.8\%} \\
     & \multirow{2}{*}{20\%} & Base & 1.801 & 0.1078 & 86.1\% \\
     &  & \ourmethod & \textbf{1.123} & \textbf{0.0604} & \textbf{87.9}\% \\
     & \multirow{2}{*}{40\%} & Base & 1.319 & 0.0807 & 88.1\% \\
     &  & \ourmethod & \textbf{0.587} & \textbf{0.0394} & \textbf{90.1\%} \\
    \hline
    \multirow{8}{*}{WikiText-103}
     & \multirow{2}{*}{5\%} & Base & 1.654 & 0.1217 & 85.8\% \\
     &  & \ourmethod & \textbf{1.532} & \textbf{0.0565} & \textbf{90.9\%} \\
     & \multirow{2}{*}{10\%} & Base & 6.568 & 0.1589 & 84.8\% \\
     &  & \ourmethod & \textbf{1.242} & \textbf{0.0498} & \textbf{91.5\%} \\
     & \multirow{2}{*}{20\%} & Base & 1.703 & 0.1493 & 84.5\% \\
     &  & \ourmethod & \textbf{1.381} & \textbf{0.0591} & \textbf{90.8\%} \\
     & \multirow{2}{*}{40\%} & Base & 1.447 & 0.1437 & 86.3\% \\
     &  & \ourmethod & \textbf{1.269} & \textbf{0.0481} & \textbf{91.7\%} \\
    \hline
    \end{tabular}
  \end{table*}

\subsection{Additional KV keep-ratio curves}
Figure~\ref{fig:qwen0p5_ppl_plots} extends the main-body optimization curves in Figure~\ref{fig:qwen0p5_ppl_plots_main} with three additional KV keep ratios: 15\%, 25\%, and 35\%.

\begin{figure}[H]
    \centering
    \includegraphics[width=\linewidth]{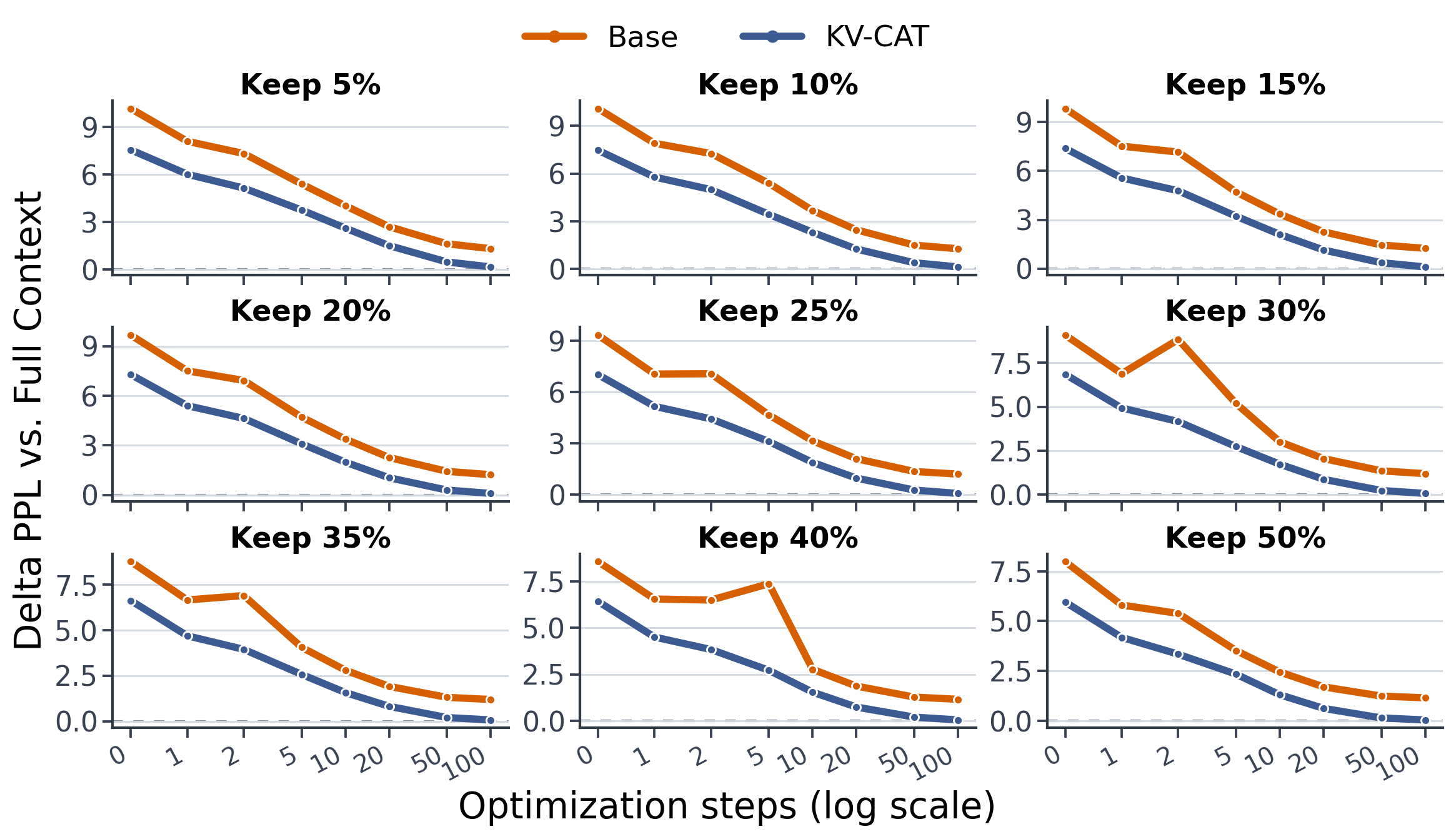}
    \caption{\textbf{\ourmethod speeds up gradient-based KV cache compression.} We plot the gap in suffix perplexity under full/compressed-prefix inference throughout gradient-based KV cache optimization. Each panel fixes a different KV keep ratio. Across ratios, the \ourmethod checkpoint achieves a comparable $\Delta$PPL in fewer optimization steps than the base model, yielding up to a \textbf{5}$\times$ speedup.}\label{fig:qwen0p5_ppl_plots}
    \vspace{-4mm}
    \label{fig:full-opt}
\end{figure}

%% file: appendix/limitations.tex
\section{Limitations}
\label{app:limitations}

Our approach relies on continued pretraining with additional routing modules, introducing non-trivial training overhead and potentially limiting applicability in low-resource settings. The routing mechanism and training objective also add implementation complexity, which may hinder integration into existing production systems. Finally, our evaluation is limited to a small set of model sizes and tasks, and it remains to be seen how well the method generalizes to larger-scale models or different domains. 

%% file: appendix/broader_impact.tex
\section{Broader Impact}
\label{app:broader_impact}
\paragraph{Broader Impact.}
This work is primarily methodological, aiming to improve the efficiency of transformer-based language models by enabling more effective KV cache compression. On the positive side, our approach can reduce memory and compute requirements for long-context inference, making large models more accessible and energy-efficient, and enabling their deployment in resource-constrained settings. Potential negative impacts are indirect: by improving the efficiency of existing LLMs, our method may lower the cost of deploying such models at scale, which could amplify both beneficial and harmful downstream applications. However, our method does not introduce new capabilities or alter model behavior beyond improving compatibility with compression, and therefore does not create new misuse vectors beyond those already present in the underlying models. Mitigation strategies largely align with existing practices for responsible LLM deployment, including monitoring, access control, and adherence to usage policies of the underlying models.